\documentclass{article}

\usepackage{PRIMEarxiv}
\usepackage{algorithm}
\usepackage{setspace}
\usepackage{algcompatible}

\usepackage[utf8]{inputenc} 
\usepackage[T1]{fontenc}    

\usepackage{hyperref}       
\usepackage{url}            
\usepackage{booktabs}       
\usepackage{amsfonts}       
\usepackage{nicefrac}       
\usepackage{microtype}      
\usepackage{lipsum}
\usepackage{fancyhdr}       
\usepackage{graphics}       
\graphicspath{{media/}}

\usepackage{subfigure}
\usepackage{booktabs} 
\usepackage{paralist}
\usepackage[numbers]{natbib}
\usepackage{hyperref}

\usepackage{xcolor}         
\usepackage{paralist}
\usepackage{graphicx}

\usepackage{multirow}

\usepackage{amsmath}
\usepackage{amssymb}
\usepackage{mathtools}
\usepackage{amsthm}

\usepackage[capitalize,noabbrev]{cleveref}

\newtheorem{theorem}{Theorem}[section]

\newtheorem{lemma}[theorem]{Lemma}

\newtheorem{definition}[theorem]{Definition}

\newtheorem{condition}[theorem]{Condition}
\definecolor{ed}{RGB}{225,0,0}

\newcommand{\ep}{\varepsilon}
\newcommand{\R}{\mathbb{R}}
\newcommand{\A}{\mathcal{A}}
\newcommand{\X}{\mathcal{X}}
\newcommand{\Y}{\mathcal{Y}}
\newcommand{\bitem}{\begin{itemize}}
\newcommand{\eitem}{\end{itemize}}
\newcommand{\benum}{\begin{enumerate}}
\newcommand{\eenum}{\end{enumerate}}
\newcommand{\beq}{\begin{equation}}
\newcommand{\eeq}{\end{equation}}
\newcommand{\beqs}{\begin{equation*}}
\newcommand{\eeqs}{\end{equation*}}
\newcommand{\bas}{\begin{align*}}
\newcommand{\eas}{\end{align*}}
\newcommand{\be}{\textnormal{Bern}}
\usepackage[textsize=tiny]{todonotes}

\title{Fair Bayes-Optimal Classifiers Under Predictive Parity}

\author{
  Xianli Zeng\\
  Shenzhen Research Institute of Big Data\\ University of Pennsylvania\\
  Philadelphia, PA\\
  \texttt{zengxl19911214@gmail.com} \\
   \And
  Edgar Dobriban \\
  University of Pennsylvania\\
  Philadelphia, PA\\
  \texttt{dobriban@wharton.upenn.edu} \\
     \And
  Guang Cheng \\
 University of California, Los Angeles\\
  Los Angeles, CA\\
  \texttt{guangcheng@ucla.edu} \\
}
\allowdisplaybreaks[4]

\begin{document}
\maketitle

\begin{abstract}
Increasing concerns about disparate effects of AI have motivated a great deal of work on fair machine learning. 
Existing works mainly focus on independence- and separation-based measures (e.g., demographic parity, equality of opportunity, equalized odds), while sufficiency-based measures such as predictive parity are much less studied. 
This paper considers predictive parity, which requires equalizing the probability of success given a positive prediction among different protected groups. 
We prove that, if the overall  performances of different groups vary only moderately,
all fair Bayes-optimal classifiers under predictive parity are group-wise thresholding rules.
Perhaps surprisingly, this may not hold if group performance levels vary widely; 
in this case 
we find that predictive parity among protected groups may lead to within-group unfairness. 
We then propose an algorithm we call FairBayes-DPP\footnote{Codes for FairBayes-DPP are available at  \url{https://github.com/XianliZeng/FairBayes-DPP}.}, aiming to ensure predictive parity when our condition is satisfied.
FairBayes-DPP is an adaptive thresholding algorithm that aims to achieve predictive parity, while also seeking to maximize test accuracy.  
We provide supporting experiments conducted on synthetic and empirical data.
\end{abstract}

\section{Introduction}
Due to an increasing ability to handle massive data with extraordinary model accuracy, machine learning (ML) algorithms have achieved remarkable success in many applications, such as computer vision \cite{KA2014VGG,SLJS2015GOOGLE,HZRS15Res,SVI2016INC} and natural language processing \cite{SVQ2014seq2,VSPU2017Attention,DCL2018BERT,YDY2019XLNET}.
However, empirical studies have also revealed that ML algorithms may incorporate bias from the training data into model predictions. 
Due to historical biases, vulnerable groups are often under-represented in available data \cite{JJSL2016, ZWY2017, SEC2019}. 
As a consequence, without fairness considerations,  ML algorithms can be systematically biased against certain groups defined by protected attributes such as race and gender.

As algorithmic decision-making systems are now widely integrated in high-stakes decision-
making processes, such as in healthcare \cite{MQ2017} and criminal prediction \cite{JJSL2016}, 
fair machine learning has grown rapidly over the last few years into a key area of trustworthy AI. 
A main task in fair machine learning is to design efficient algorithms satisfying fairness constraints with a small sacrifice in model accuracy. 
This field has made substantial progress in recent years, as many effective approaches have been proposed to mitigate algorithmic bias \cite{ZWKT2013,CKYMR2016,CFDV2017,XWYZW2019,LV2019,CJG2019,DETR2018,CMJW2019,JL2019,CHS2020,ZDC2022}.

An important fundamental benchmark for fair classification is provided by fair Bayes-optimal classifiers, which maximize accuracy subject to fairness  \cite{MW2018,ZDC2022}. 
A key class of classifiers is group-wise thresholding rules (GWTRs) over the feature-conditional probabilities of the target label, for each protected group (e.g., probability of repaying a loan given income). 
Intuitively, being a GWTR is a minimal requirement for within-group fairness: the most qualified individuals are selected in every group.
 \cite{CSFG2017,MW2018,CCH02019,I2020,SC2021,ZDC2022} have studied fair Bayes-optimal classifiers  under various fairness constraints and proved that,
for many  fairness metrics, the optimal fair classifiers are GWTRs. Moreover, the associated thresholds can be learned efficiently \cite{MW2018,ZDC2022}.

Current literature on Bayes-optimality focuses mainly on the independence- and separation-based fairness measures (e.g., demographic parity, equality of opportunity, equalized odds; see Section \ref{fm} for definitions and a review). 
However, sufficiency-based measures such as predictive parity are less commonly considered, possibly due to the complexity of their constraints. 
Sufficiency-based measures are often applied to assess recidivism prediction instruments \cite{FBC2016,DMB2016,C2017fair}. 
\cite{LSMH2019} show that a particular sufficiency-based measure, group calibration,
is implicitly favored by unconstrained optimization: 
calibration error is bounded by the excess risk over the unconstrained Bayes-optimal classifier. 
For selective classification, \cite{LBRSP2021} find that sufficiency-based representation learning leads to fairness. 
Despite these findings, little is known about (1) what are the optimal fair classifiers under sufficiency-based measures and (2) how to learn them.
In this paper, we aim to answer these two questions. 
We consider predictive parity, which requires that the positive predictive value (probability of a successful outcome given a positive prediction) be similar among protected groups.
In credit lending, for example,  predictive parity requires that, for individuals who receive the loans, the repayment rates in different protected groups are the same.

We study fair Bayes-optimal classifiers under predictive parity. 
Perhaps surprisingly, our theoretical results reveal that  the  optimal fair classifiers may or may not be a GWTR, depending on the data distribution. 
We identify a sufficient condition under which all fair Bayes-optimal classifier are GWTRs. 
Without this condition, 
we show that fair Bayes-optimal classifiers may not be a GWTR when the minority group is more qualified than the majority group. 
In these cases, predictive parity may have limitations as a fairness measure, as it can lead to within-group unfairness for the minority group. 
Our findings are a reminder that the improper use of fairness measures may result in severe unintended consequences. 
Careful analysis before applying fairness measures is necessary.

We then develop an algorithm, FairBayes-DPP, aiming for predictive parity. 
Our method is a two-stage plug-in method. 
In the first step, we use standard learning algorithms to estimate group-wise conditional probabilities of the labels. 
In the second step, we first check our sufficient condition. If the sufficient condition holds,
we apply 
a plug-in method 
for estimating the optimal thresholds under fairness for each protected group.

We summarize our contributions as follows.
\begin{itemize}
    \item We show that Bayes-optimal classifiers satisfying predictive parity may or may not be group-wise thresholding  rules (GWTRs), depending on the data distribution.
    
    \item We identify a sufficient condition under which all fair Bayes-optimal classifiers are GWTRs. 
    However, when the sufficient condition is not satisfied, the fair Bayes-optimal classifier may lead to within-group unfairness for the minority group.
    
    \item We propose the  FairBayes-DPP algorithm  for binary fair classification.  
    The proposed  FairBayes-DPP  is computationally efficient, showing a solid performance in our experiments.
\end{itemize}

\section{Related Literature}
\subsection{Fairness Measures}
\label{fm}

Various fairness metrics have been proposed to measure aspects of disparity in ML. 
Group fairness \citep{CFM2009,DHPT2012,HPS2016} targets statistical parity across
 protected groups, while individual fairness \citep{JKMR2016,PKG2019,RBMF2020} aims to provide nondiscriminatory predictions
for similar individuals.  
In general, group fairness measures can be categorized into three categories. 

The first group consists of independence-based measures, which require independence between predictions and protected attributes; this includes demographic parity \cite{KC2012,ZWKT2013} and conditional statistical parity \cite{CSFG2017,I2020}. In  credit lending, independence means that the proportion of approved candidates is the same across different protected groups.
However, as discussed in \cite{HPS2016}, independence-based measures have limitations; 
and applying them often leads to a substantial loss of accuracy.

The second group consists of separation-based measures,
which require conditional independence between predictions and protected attributes, given label information. 
Typical examples in this group are equality of opportunity \cite{HPS2016,ZLM2018} and equalized odds \cite{HPS2016,ZVGGK2017}.  
In  credit lending, separation-based measures require, that the individuals  who will  pay back (or default on) their loan 
have an equal probability of
getting the loan, despite their race or gender. 
Compared to independence-based measures, separation-based measures take label information into account, allowing for perfect predictions that equal the label. 
However, these measures are hard to validate in certain applications as the label information is often unknown for some groups. For example, the repayment status is missing for individuals whose loan application is declined.

As a result, measuring predictive bias is more widely applicable. 
This leads to the third class, sufficiency-based measures \cite{PMK2017,C2017fair,LSMH2019}, where the label is required to be conditionally independent of the protected attributes, given the prediction. In credit lending, this requires that among the approved applications, the proportion of individuals who pay back the loan is equal across different groups.
Unlike independence- and separation-based measures that are well studied with solid theoretical benchmarks and efficient algorithms, sufficiency-based measures are less commonly investigated. 
A possible reason is that conditioning on the prediction
leads to a complex constraint, which is thus challenging to study and enforce algorithmically.

\subsection{Algorithms Aimed at Fairness}
Literature on algorithms for fairness has grown explosively over the past decade.
Existing algorithms for fairness can be categorized broadly into three categories.
The first category is pre-processing algorithms aiming to remove biases from the training data. 
Examples include transformations \citep{FFMS2015,KJ2016,CFDV2017,JL2019}, fair representation learning \citep{ZWKT2013,CKYMR2016,DETR2018,CMJW2019} and fair data generation \citep{XYZ2018,SHC2019,XYfairplus,RKSR2021}. 
The second group is 
in-processing algorithms, which handle fairness constraints during the training process. 
Two common strategies are penalized optimization \citep{GCAM2016,NH2018,CJG2019,CHS2020}  and adversarial training \citep{ZLM2018,CFC2018,XWYZW2019,LV2019}. 
The former incorporates fairness measures as a regularization term into the optimization objective and
the later tries to minimize
the predictive ability of the model with respect to 
the protected attribute. 

The third group is post-processing algorithms, aiming to remove disparities from the model output. 
The most common post-processing algorithm is the thresholding method \citep{BJA2016,MW2018,I2020,SC2021,ZDC2022}, adjusting thresholds for every protected group to achieve fairness.  In this paper, we propose a post-processing algorithm, FairBayes-DPP, to estimate the fair Bayes-optimal classifier under predictive parity.

\section{Problem Formulation and Notations}\label{Pre_and _Not}
In this paper, we consider 
classification problems where two types of feature are observed: the usual feature $X\in\mathcal{X}$, and the  protected feature $A\in\mathcal{A}$. 
For example, in loan applications, $X$ may refer to common features such as 
 education level and income, and $A$ may correspond to the race or gender of a candidate. 
As multiclass protected attributes are often encountered in practice, we allow $\mathcal{A}$ to have any number $|\mathcal{A}|\ge1$ of classes, and let $\mathcal{A}=\{1,2,...,|\mathcal{A}|\}$.
We denote by $Y$ the ground truth label. 
In credit lending, $Y$  may correspond to the status of repayment or defaulting on a loan.
The output $\hat{Y}$ of the classifier aims to predict $Y$ based on observed features.  
We consider randomized classifiers defined as follows:
\begin{definition}[Randomized classifier]
  A randomized classifier is a measurable function\footnote{We assume that, whenever needed, the sets considered are endowed with appropriate sigma-algebras, and all functions considered are measurable with respect to the appropriate sigma-algebras.}
  $f:\mathcal{X}\times\mathcal{A} \to [0,1]$, indicating the probability of predicting $\widehat{Y}=1$ when observing $X=x$ and $A=a$. We denote by $\hat{Y}_f=\hat{Y}_f(x,a)$ the prediction induced by the classifier $f$.
 \end{definition}
Group-wise thresholding rules (GWT rules/classifiers or GWTRs) over conditional probabilities are of special importance.
Consider an appropriate dominating sigma-finite measure $\mu$ on $\mathcal{X}$ (such as the Lebesgue measure for measurable subsets of $\R^d$, $d\ge 1$, or the uniform measure for finite sets), and suppose that 
for all $a\in\A$ and $y\in \Y$,
the features $X$ have a conditional distribution $P_{X|a,y}$
given $A=a,Y=y$
with a density $dP_{X|a,y}$
with respect to $\mu$.
For all\footnote{To be precise, this conditional density is defined for $\mu$-almost every $x\in\X$; however for simplicity we say for all $x\in\X$. We use this convention without further mentioning through the paper.}
$x\in\X$
and $a\in\A$, let $\eta_a(x)=P(Y=1|X=x,A=a)$.  
 \begin{definition}[GWT classifier]
  A classifier $f$ is a GWTR if there are constants $t_a$, $a\in\A$, and functions $\tau_a:\X\to[0,1]$, $a\in\A$,  such that for all $x\in\mathcal{X}$ and $a\in\mathcal{A}$
  \begin{equation}\label{gwtr}
      f(x,a)=I(\eta_a(x)>t_a)+\tau_a(x) I(\eta_a(x)=t_a),
  \end{equation}
 where $I(\cdot)$ is the indicator function.
 \end{definition}

Clearly, GWTRs choose individuals with the highest conditional probability in each group. 
This property is a minimal requirement for within-group fairness.
For example, a GWT recruitment tool ensures that the most qualified candidates are approved in every protected group. 

We consider predictive parity, which aims to ensure the same positive predictive value among protected groups:
  \begin{definition}[Predictive Parity]
  A classifier $f$ satisfies predictive parity if for all $a\in\mathcal{A}$,
   $$P(Y=1 |\widehat{Y}_f  = 1,A = a) =P(Y=1|\widehat{Y}_f  = 1).$$
  \end{definition}
We follow \cite{CHS2020, ZDC2022} to use the difference between positive predictive values to measure the degree of unfairness, defining the Difference in Predictive Parities (DPP) of a classifier $f$ as
$$\text{DPP}(f)=\sum_{a\in\mathcal{A}}|P(Y=1 |\widehat{Y}_f  = 1,A = a) -P(Y=1|\widehat{Y}_f  = 1)|.$$



\section{Fair Bayes-optimal Classifiers under Predictive Parity}
Since predictive parity is commonly considered under the scenarios where false positives are particularly harmful \cite{LBRSP2021},
we study cost-sensitive  classification. 
For
a cost parameter $c\in[0,1]$\footnote{When $c=1/2$, cost-sensitive risk reduces to the usual zero-one risk.}, the cost-sensitive 0-1 risk of the classifier $f$ is defined as
$$R_c(f)=c\cdot P(\hat{Y}_f=1,Y=0)+(1-c)\cdot P(\hat{Y}_f=0, Y=1).$$
An unconstrained Bayes-optimal classifier for the cost-sensitive risk is any minimizer $f^\star\in \text{argmin}_f R_c(f).$  
A classical result is that all Bayes-optimal classifiers have the form 
$f^\star(x,a)=I(\eta_a(x)>c)+\tau I(\eta_a(x)=c)$, where $\tau\in[0,1]$ is arbitrary \cite{C2001CS, MW2018}. 
Taking  predictive parity into account, a fair Bayes-optimal  classifier is any  minimizer of the cost-sensitive risk among  fair classifiers:
\begin{equation}\label{bao}
    f_{PPV}^\star\in \underset{f: \text{DPP}(f)=0}{\text{argmin}} R_c(f).
\end{equation}

\subsection{GWT Fair Bayes-Optimal Classifiers under Predictive Parity}
We first  identify a sufficient condition under which  all fair Bayes-optimal classifier under predictive parity are GWTRs. 

\begin{condition}[Sufficient condition for Bayes-optimal classifiers to be GWTRs]\label{suff-con} 
\begin{equation*}
\underset{a\in\mathcal{A}}{\min}\,
P(Y=1|\,\eta_a(X)\geq c,A=a)
\geq\underset{a\in\mathcal{A}}{\max}\, P({Y=1}|A=a). \end{equation*}
\end{condition}
The sufficient condition \ref{suff-con} requires that the minimal group-wise positive predictive value 
$P(Y=1|\,\eta_a(X)\geq c,A=a)$
of the unconstrained Bayes-optimal classifier  is lower bounded by the maximal proportion of positive labels $P({Y=1|A=a})$ among groups. 
In other words, the  performances of different groups vary only moderately:  
the average performance of the most qualified class of each group---the points with $x$ such that $\eta_a(x)\geq c$---should be better than the overall performance $P({Y=1|A=a})$ of any of the other groups.  
Condition \ref{suff-con} holds if
$P({Y=1|A=a})\leq c$ for all $a\in\mathcal{A}$, 
because $P(Y=1|\,\eta_a(X)\geq c,A=a)\geq c$. 

These conditions are  applicable in settings where $c$ is large, such as in credit lending where false positives are more harmful than false negatives, or if $p_{Y|a}$, $a\in\mathcal{A}$ are small, such as in job recruitment or school admissions where the number of slots is much smaller than the number of applications.  
Under this condition, we present our main result.

\begin{theorem}[Main result]\label{thm-GWTR}
Consider the cost-sensitive $0$-$1$ risk with cost parameter $c$. If
Condition \ref{suff-con} holds,
then all fair Bayes-optimal classifiers under predictive parity are GWTRs. 
Thus, for all $f_{PPV}^\star$ from \eqref{bao}, 
there are $(t_a)_{a=1}^{|\mathcal{A}|} \in [0,1]^{|\mathcal{A}|}$ and functions $\tau_a(x):\X\to[0,1]$
such that \eqref{gwtr} holds.

\end{theorem}

Unlike for demographic parity or for equality of opportunity, 
where the fairness constraint is linear with respect to the probability predictions of the classifier $f$ \cite{MW2018}, the DPP constraint is non-linear with respect to $f$. 
As a consequence, previously used theoretical tools such as the Neyman-Pearson argument from hypothesis testing \cite{ZDC2022} are no longer valid in this case. 
Instead, we prove the result using a novel constructive argument. 
When Condition \ref{suff-con} is satisfied, 
for any classifier satisfying predictive parity, which is not a GWTR, we construct a GWTR that satisfies predictive parity and achieves a smaller classification error.
As a result, under Condition \ref{suff-con}, all fair Bayes-optimal classifiers are GWTRs. 
Overall, the proof of Theorem \ref{thm-GWTR} is quite involved, and requires a lot of careful casework and analysis.

\subsection{Fair Bayes-optimal Classifiers under Predictive Parity do not Need to be Thresholding Rules}
Next, we consider the case when the sufficient condition \ref{suff-con} does not hold. 
For simplicity, we consider a binary protected attribute $a\in\{0,1\}$ with 
\begin{equation}\label{unsuff-con}
{P(Y=1|\,\eta_a(x)\geq c,A=1)}<P(Y=1|A=0). \end{equation}

Our result shows that, under condition \eqref{unsuff-con}, there exist class probabilities $p_a$, $a\in\mathcal{A}$, such that no Bayes-optimal classifier under predictive parity is a GWTR.
\begin{theorem}\label{thm-imGWT}
Suppose that condition \eqref{unsuff-con} holds. 
Denote $t_1=\inf\{t: P(Y=1|\,\eta_1(X)\geq t,A=1)>P({Y=1|A=0})\}$. Suppose there exist $\delta_1,\delta_2>0$ such that $P(c+\delta_1<\eta_A(X)<t_1|A=1)=\delta_2>0$. 
Then, for all $p_1>\frac{2}{2+\delta_1\delta_2}$, no fair Bayes-optimal classifier under predictive parity is a GWTR.

\end{theorem}

The condition involving the constants $\delta_1,\delta_2>0$ ensures that $\eta_1(X)$ has positive probability to be strictly larger than $c$, which is a technical condition needed in the proof. 
Theorem \ref{thm-imGWT} shows that predictive parity may lead to within-group unfairness, whereby the most qualified individuals are predicted to be unqualified, for a better overall accuracy.
By definition, predictive parity requires that the qualifications of selected individuals are similar across the protected groups. Suppose there exists a highly qualified minority group in which most individuals are qualified. 
Selecting the most qualified individuals in this group leads to a very high standard.
As a result, many qualified individuals in other majority groups may be predicted to be unqualified using this standard, 
leading to accuracy loss. 
Conversely, if we select less qualified individuals in the highly qualified group, the lower standard allows more qualified individuals from the other groups to be selected, and increases accuracy.

\section{FairBayes-DPP: Adaptive Thresholding for Fair Bayes-optimality}\label{sec-alg}
In this section, we propose the FairBayes-DPP algorithm (Algorithm \ref{alg:pre-par}) for fair Bayes-optimal classification under predictive parity.
As mentioned, the DPP constraint is non-linear with respect to the classifier $f$, and is also highly non-convex with respect to the model parameters, even if both the classifier $f$  and the risk function are convex with respect to these parameters. 
In such cases, 
incorporating fairness constraints as a penalty in the training objective may be hard due to potential local minima. 
Therefore, we consider a different approach, developing a new two-step plug-in method based on Theorem \ref{thm-GWTR}.
Suppose we observe data points $(x_i,a_i,y_i)_{i=1}^n$ drawn independently and identically from a distribution $\mathcal{D}$ over the domain
$\mathcal{X} \times \mathcal{A} \times \mathcal{Y}$.

\textbf{Step 1.}
In the first step, we apply standard machine learning algorithms to learn the feature- and group-conditional label probabilities $\eta$ based on the whole dataset.
Consider  a loss function $L(\cdot, \cdot)$ and the function class $\mathcal{F}=\{f_\theta: \theta\in\Theta \}$ parametrized by $\theta$. 
The estimator of $\eta$ is obtained by minimizing the empirical risk, $\hat\eta_a(x):=f_{\hat\theta}(x,a)$, where
\begin{align}\label{ht}
    \hat{\theta}&\in \underset{\theta\in\Theta}{\text{argmin}}\frac1{n}\sum_{i=1}^nL(y_{i},f_{\theta}(x_{i},a_i)).
\end{align}
Here we use the cross-entropy loss, as minimizing the empirical 0-1 risk is generally not tractable. 
At the population level,
the minimizers of the risks induced by the 0-1 and cross-entropy losses are both the
true conditional probability function \citep{miller1993loss}.
\begin{algorithm}[tb]
   \caption{FairBayes-DPP}
   \label{alg:pre-par}
\begin{algorithmic}
   \STATE {\bfseries Input:}  Datasets $S=\cup_{a=1}^{|\mathcal{A}|}S_a$ with $S=\{x_{i},a_i,y_{i}\}_{i=1}^{n}$ and $S_a=\{x^{(a)}_{j},y^{(a)}_{j}\}_{j=1}^{n_a}$. Cost parameter $c\in[0,1]$.

   \STATE {\textbf{Step 1}:} Estimate $\eta_a(x)$ by $\hat\eta = f_{\hat\theta}$, with $\hat\theta$ from \eqref{ht}
           \STATE {\textbf{Step 2}: Find the optimal thresholds.}.
   \STATE  Define, for all $t$,\qquad
    $\widehat{\text{PPV}}_{a}(t)=\frac{{\sum}_{j=1}^{n_a} I(y^{(a)}_j=1,\hat{\eta}_{a}(x^{(a)}_j)\geq t)}{{\sum}_{j=1}^{n_a}I(\hat{\eta}_{a}(x^{(a)}_j)\geq t)}$, \qquad $\hat{P}(Y=1|A=a)=\frac{1}{n_a}\sum\limits_{i=1}^{n_a}y^{(a)}_j$.
    
  \IF{$\min_a\widehat{\text{PPV}}_{a}(c)< \max_a\hat{P}(Y=1|A=a)$}
  \STATE {We recommend considering other fairness measures.}
  \ELSE
  \STATE Let {$t_{\min} = \min\{t: 
 \widehat{\text{PPV}}_{1}(t)\ge \max_a\hat{P}(Y=1|A=a)
  \}.$}
  \FOR{$t \in \mathcal{T}=[t_{\min}, \max_j\hat\eta_1(x^{(1)}_{j})]$}
    \FOR{$a\in\mathcal{A} \setminus\{1\}$}
    \STATE Find $\hat{t}_a(t)$ such that   {$\widehat{\text{PPV}}_{a}(\hat{t}_a(t))\approx\widehat{\text{PPV}}_1(t).$}
    \ENDFOR  
    \STATE Let $
\hat{f}(x,a,t)
=\tilde{f}\left(x,a;\hat t_1(t),\hat t_2(t),...,\hat t_{|\mathcal{A}|}(t)\right)
=I\left(\widehat\eta_{a}(x)\geq \hat t_a(t)\right).$
\STATE Let $R_c(t)=\frac1n\sum_{i=1}^n c^{(1-y_i)}(1-c)^{y_i}I(y_i\neq \hat{f}(x_i,a_i,t)).
$
    \ENDFOR 
    \STATE Find $\hat{t}=\underset{t\in \mathcal{T}_n}{\text{argmin}}R_c(t).$

\STATE {\bfseries Output:}  
$\widehat{f}_{PP}(x,a)=I(\widehat\eta_{a}(x)\geq\hat{t}_a(\hat{t}))$
    \ENDIF
    \end{algorithmic}
\end{algorithm}

\textbf{Step 2.}
In the second step, we first check the empirical version of Condition \ref{suff-con} for the classifier derived in the first step. 
To be more specific, we divide the data into $|\mathcal{A}|$ parts,
according to the value of $A$: for $a\in\mathcal{A}$, 
$S_a=\{x^{(a)}_{j},y^{(a)}_{j}\}_{j=1}^{n_a}$, where $a^{(a)}_{j}=a$. 
Let, for all $t$ for which it is defined, $$\widehat{\text{PPV}}_{a}(t)=\frac{{\sum}_{j=1}^{n_a} I(y^{(a)}_j=1,\hat{\eta}_{a}(x^{(a)}_j)\geq t)}{{\sum}_{j=1}^{n_a}I(\hat{\eta}_{a}(x^{(a)}_j)\geq t)} \qquad \text{ and } \qquad \hat{P}(Y=1|A=a)=\frac{1}{n_a}\sum\limits_{i=1}^{n_a}y^{(a)}_j.$$
We only divide by nonzero quantities here and below. 
To ensure that the quantities we divide by are nonzero, we restrict to  $t_a\in[0,\max_j(\hat\eta_a(x^{(a)}_{j}))]$ when evaluating $\widehat{\text{PPV}}_{a}(t_a)$.
We check whether $\min_{a}\, \widehat{\text{PPV}}_{a}(c)\geq\max_a\, \hat{P}(Y=1|A=a) $.\footnote{One could modify this to allow some slack; and perform a formal statistical hypothesis test of our sufficient condition.}
If this is not satisfied, we recommend considering other fairness measures, as predictive parity may not be appropriate in this case, see the discussion after Theorem \ref{thm-imGWT}. 
 If it is satisfied,
 we then  adjust the thresholds of the classifier aiming for predictive parity. 
Based on Theorem~\ref{thm-GWTR}, we consider the following  deterministic classifiers:
\begin{equation}\label{classifier1}
\tilde{f}(x,a;t_1,t_2,...,t_{|\mathcal{A}|})=I\left(\widehat\eta_{a}(x)\geq t_a\right),
\end{equation}
where $\hat{\eta}$ is the estimate of $\eta$ from the first step, and $t_a$, $a\in\mathcal{A}$, are parameters to learn. 

We use the following strategy to estimate $t_a$, $a\in\A$: First, we fix the threshold for the group with $a=1$, say $t$. 
The positive predictive value for this group 
can then be estimated by $\widehat{\text{PPV}}_{1}(t)$.
To achieve predictive parity, we need to find thresholds for the other groups such that the positive group-wise predictive values are the same\footnote{Since a sample mean $n^{-1}\sum_{i=1}^n Z_i$ of iid random variables $Z_i$ has a variability of order $O_P(n^{-1/2})$, even if the true predictive parities are equal, the empirical versions may differ by $O_P(n^{-1/2})$. However, in our case we simply find the values $t_a,t$ for which they are as close as possible.}, i.e., find $t_a$, $a=2,3,\ldots,|\A|$, such that
\begin{equation}\label{thresh-estimate}
\widehat{\text{PPV}}_{a}(t_a)
\approx
\widehat{\text{PPV}}_{1}(t)
, \ \ \ \text{ for } a = 2,3,...,|\mathcal{A}|.
\end{equation}
As stated in Lemma \ref{lem1}, the positive predictive value for each group in the population is always non-decreasing with the thresholds $t_a$ increases. 
As a consequence, we can search over
$t_a$, $a=2,3,\ldots,|\A|$, 
efficiently  via, for instance, the bisection method.\footnote{The empirical PPV is only approximately monotonic, but this does not cause problems.} 
Correspondingly, we consider the following range of $t$: $\mathcal{T}=[t_{\min},\max_j \eta_1(x^{(1)}_{j})]$ with
$$t_{\min} = \min\{t: 
\widehat{\text{PPV}}_{1}(t)\ge \max_a\hat{P}(Y=1|A=a)
\}.$$

We denote by $\hat t_a(t)$, $a=2,3,\ldots,|\A|$, the estimated thresholds given by \eqref{thresh-estimate}, writing $\hat t_1(t)=t$ for convenience. 
We consider the classifier \eqref{classifier1} with these thresholds:
\begin{equation*}
\hat{f}(x,a,t)
=\tilde{f}\left(x,a;\hat t_1(t),\hat t_2(t),...,\hat t_{|\mathcal{A}|}(t)\right)
=I\left(\widehat\eta_{a}(x)\geq \hat t_a(t)\right).
\end{equation*}
 Lastly, we find $t$ that minimizes the cost-sensitive risk on the training data by searching over a grid $\mathcal{T}_n$ within $\mathcal{T}$:
$$
\hat{t}=\underset{t\in \mathcal{T}_n}{\text{argmin}}\left\{\frac1n\sum_{i=1}^n c^{(1-y_i)}(1-c)^{y_i}I(y_i\neq \hat{f}(x_i,a_i,t))\right\}.
$$
Our final estimator of the fair Bayes-optimal classifier is $\hat{f}_{PP} = \hat f_{\hat t}$. 
The FairBayes-DPP algorithm is related to the algorithms proposed for other fairness measures in  \cite{ZDC2022}, where a binary protected attribute is considered and closed-form optimal thresholds are derived. 
In contrast, FairBayes-DPP can handle multi-class protected attributes and does not rely on closed-form thresholds.
Similar to \cite{ZDC2022}, our algorithm enforces fairness 
only in the fast second step, where no gradient-based technique is applied. 
Thus, it is computationally efficient and the non-convexity of fairness constraint is no longer problematic. 
Our experimental results demonstrate that our method removes disparities and preserves accuracy.

\section{Experiments}

\begin{table}
\caption{Classification accuracy and DPP of the true fair Bayes-optimal classifier and our estimator trained via logistic regression on a synthetic data example. See Section \ref{synth} for details.}
\label{table_synth}
\vspace{-0.2cm}
\begin{center}
\setlength{\tabcolsep}{5.7pt}
\renewcommand{\arraystretch}{0.95}
\begin{small}
\begin{sc}
\begin{tabular}{cc|cc|cc|cc}
\hline
\multicolumn{4}{c|}{Theoretical Value}  & \multicolumn{4}{c}{Logistic regression} \\\hline
\multicolumn{2}{c|}{Fair}  &\multicolumn{2}{c|}{Unconstrained}  &\multicolumn{2}{c|}{FairBayes-DPP}  &\multicolumn{2}{c}{Unconstrained}   \\\hline
  $p$ & ACC &DPP& ACC & DPP& ACC&{DPP}& {{ACC}} \\\hline
0.2 & 0.814 & 0.000 & 0.814 & 0.049\,(0.036) & 0.813\,(0.005) & 0.046\,(0.037) & 0.813\,(0.005)\\
0.3 & 0.794 & 0.024 & 0.794 & 0.037\,(0.029) & 0.794\,(0.006) & 0.040\,(0.033) & 0.794\,(0.005)\\
0.4 & 0.781 & 0.050 & 0.781 & 0.035\,(0.029) & 0.781\,(0.006) & 0.054\,(0.029) & 0.782\,(0.005)\\
0.5 & 0.775 & 0.078 & 0.777 & 0.042\,(0.032) & 0.775\,(0.006) & 0.081\,(0.036) & 0.777\,(0.006)\\
0.6 & 0.778 & 0.113 & 0.781 & 0.038\,(0.031) & 0.778\,(0.006) & 0.113\,(0.037) & 0.781\,(0.006)\\\hline
\end{tabular}
\end{sc}
\end{small}
\end{center}
\vspace{-0.4cm}
\end{table}
\subsection{Synthetic Data}
\label{synth}
We first study a synthetic dataset to compare our method with the true Bayes-optimal fair classifier derived analytically using the true data distribution.

\textbf{Statistical model.}
Let $X=(X_1,X_2)\in \mathbb{R}^2$ be a generic feature, $A\in\{0,1\}$ be the protected attribute and $Y\in\{0,1\}$ be the label. 
We generate $A$ and $Y$ according to the probabilities $P(A=1)$, $P(Y=1|A=1)$ and $P(Y=1|A=0)$, specified below.
Conditional on $A=a$ and $Y=y$, $X$ is generated from a bivariate
Gaussian distribution $N((2a-1,2y-1)^\top,2^2I_2)$, where $I_p$ is the $p$-dimensional identity covariance matrix.
In this model,  $\eta_a(x)$ has a closed form, and we use it to find the true fair Bayes-optimal classifier numerically under the Condition \ref{suff-con}. More details about this synthetic model can be found in Section \ref{appendix-synth} of Appendix.


 \textbf{Experimental setting.}
 We randomly sample $50,000$ training data points  and $5,000$ test data points. 
 In the Gaussian case, the Bayes-optimal classifier is linear in $x$ 
and thus we employ logistic regression  to learn $\eta_1(\cdot)$ and $\eta_0(\cdot)$. 
 We then 
 search over a grid with spacings equal to $0.001$ over the range we identified in Section \ref{sec-alg}
for the empirically optimal thresholds under fairness.
 We denote  $\widehat{f}$ and $\widehat{f}_{PPV}$ the estimators of the unconstrained and fair Bayes-optimal classifiers, respectively. 

We first evaluate the FairBayes-DPP algorithm under the  Condition \ref{suff-con}. We set the cost parameter $c=0.5$, while $P(A=1)=0.3$ and $P({Y=1|A=0})=0.2$.
It can be calculated that $P(Y=1|\,\eta_0(X)>0.5,A=0)\approx 0.613$, using \eqref{suff-gaussian} in the Appendix. 
To consider settings with varied levels of fairness in the population, 
we vary $p=P({Y=1|A=1})$ from $0.2$ to $0.6$, with the DPP of unconstrained Bayes-optimal classifier grows from $0$ to $0.113$. 

Table \ref{table_synth} presents the classification accuracy and DPP of the true fair Bayes-optimal classifier and FairBayes-DPP trained via logistic regression over 100 simulations\footnote{Here, the randomness of the experiment is due to the random generation of the synthetic data.}. 
Our first observation is that, under predictive parity, the accuracy of true unconstrained and fair Bayes-optimal classifiers is almost identical, indicating  that predictive parity under Condition \ref{suff-con} requires a very small loss of accuracy. 
This finding is consistent with the results in \cite{LSMH2019} that sufficiency-based measures are favored by unconstrained learning.

Second, our FairBayes-DPP method  closely tracks the behavior of the fair Bayes-optimal classifier, controlling the accuracy metric ACC and unfairness metric DPP on the test data effectively.
When $|P({Y=1|A=1})-P({Y=1|A=0})|$ is small, FairBayes-DPP  performs similarly to the unconstrained classifier. 
However, when the data is biased against protected groups and $|P({Y=1|A=1})-P({Y=1|A=0})|$ is large, 
FairBayes-DPP  mitigates the disparity of the unconstrained classifier effectively, while preserving model accuracy. We further conduct extensive simulations to evaluate  the FairBayes-DPP algorithm with different model and training setups, as shown in the Appendix. In particular, we also consider the  multi-class protected attribute case.



\subsection{Empirical Data Analysis}
{\bf Dataset.} We test FairBayes-DPP on two  benchmark  datasets for fair classification: ``Adult'' \cite{Dua2019} and ``COMPAS'' \cite{JJSL2016}. 
For each dataset, we randomly sample (with replacement) 70\%, 50\% and 30\% as the training, validation and test set, respectively. 
To further test the performance of our algorithm on a large-scale dataset, we conduct experiments on the CelebFaces Attributes (CelebA) Dataset \cite{CelebAdata}.

\begin{compactitem}
\item {\it Adult:} The target variable $Y$ is whether the income of an individual is more than \$50,000.  Age, marriage status, education level and other related variables are included in $X$, and the protected attribute $A$ refers to gender. 

\item {\it COMPAS:} In the COMPAS dataset, the target is to predict recidivism. Here $Y$ indicates whether or not a criminal will reoffend, while $X$ includes  prior criminal records, age and an indicator of misdemeanor. 
The protected attribute $A$  is the race of an individual, ``white-vs-non-white''.

\item {\it CelebA:} CelebA dataset is a large-scale dataset with more than 200,000 face images, 
each with 40 attributes (including protected attribute ``gender'' and other 39 different attributes for prediction tasks). 
Our goal is to predict the face attributes $Y$ based on the images $X$ and remove bias with respect to gender  $A$ from the output. 
\end{compactitem}

\begin{figure}[t]
    \centering
    \includegraphics[scale=0.325]{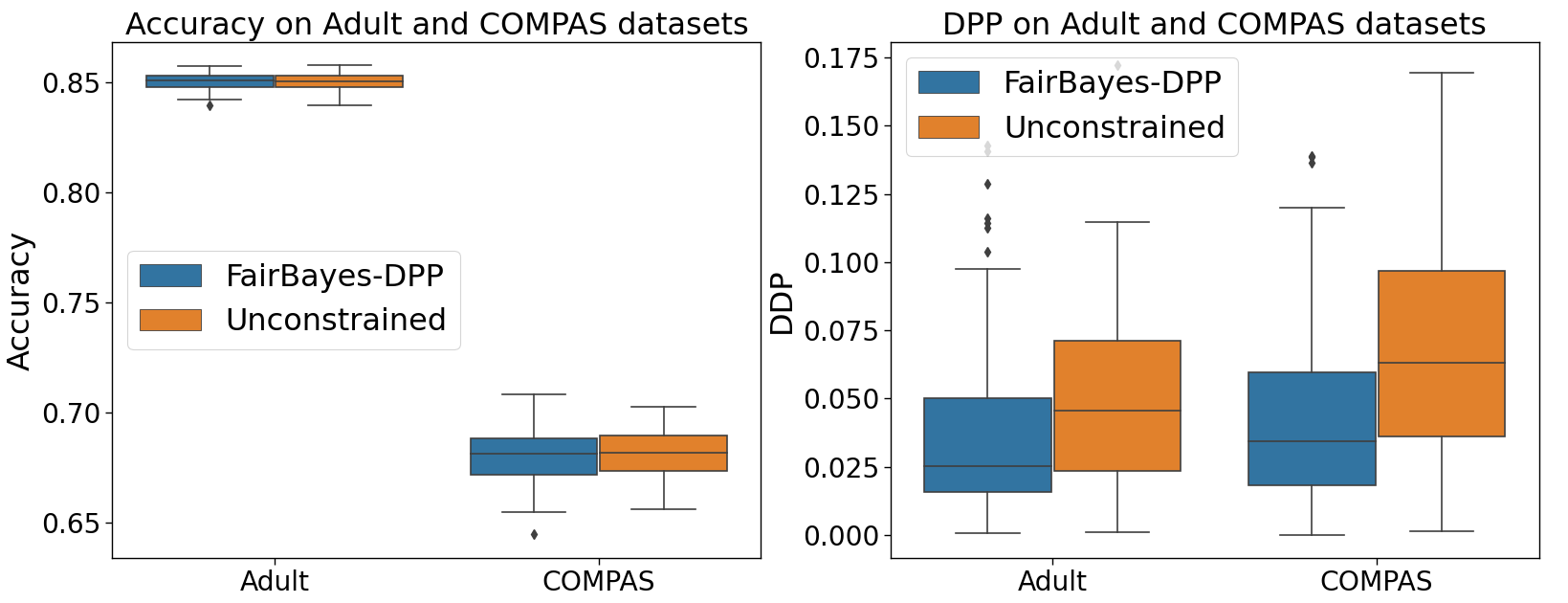}
        \vspace{-0.2cm}
    \caption{Accuracy and difference in predictive parity on the Adult and  COMPAS datasets.}
    \vspace{-0.3cm}
    \label{fig-ad_co}
\end{figure}

 \textbf{Experimental setting.} 
 As algorithms for predictive parity are rarely considered in the literature, we use unconstrained learning as a baseline for our experiments. 
 For the ``Adult'' and ``COMPAS'' datasets, we adopt the same training setting as in \cite{CHS2020, ZDC2022}. The conditional probabilities are learned via a three-layer fully connected neural network architecture with 32 hidden neurons per layer. 
 For  ``CelebA'', we apply the training setting from \cite{WQK2020}. We learn the conditional probabilities by training a ResNet50 model; \citep{HZRS15Res}, pretrained on ImageNet \cite{DDSLKL2009}. For all the datasets, 
Over the course of training the model on
the training set, we select the one with best performance on the validation set. 
In addition, we learn the optimal thresholds over the validation set to avoid overfitting.
All experiments use PyTorch. We refer
readers to the Appendix for more training details, including optimizer, learning rates, batch sizes and training epochs. 
We repeat the experiment 100 times for the Adult and  COMPAS datasets and 10 times for the CelebA dataset.\footnote{For the Adult and  COMPAS datasets, the randomness of the experiment comes from the random selection of the training, validation and test data, as well as the 
stochasticity of the batch selection in the optimization algorithm. 
For the CelebA dataset, the randomness is caused by the stochasticity of the optimization method.}

Figure \ref{fig-ad_co} presents the average performances of FairBayes-DPP and unconstrained learning on the Adult and  COMPAS datasets.
Our method achieves almost the same accuracy as the unconstrained classifier, and has a smaller disparity.  
To better compare our fair classifier with the unconstrained one, we  use the paired $t$-test to compare the DPP of the proposed algorithm ($\text{DPP}_{Fair}$) and of unconstrained learning ($\text{DPP}_{Base}$).  
We consider the following one sided test: 
$$\mathcal{H}_0: \text{DPP}_{fair}=\text{DPP}_{Base} \ \ \ \text{ v.s. }  \ \ \ \mathcal{H}_1:\text{DPP}_{fair}<\text{DPP}_{Base} .$$
The $p$-values of the tests are $3.90\times 10^{-4}$ for the Adult dataset and $3.09\times 10^{-8}$
for the COMPAS dataset. 
In both cases, these results provide evidence 
that our  FairBayes-DPP   achieves a smaller disparity than unconstrained learning.

Finally, we test FairBayes-DPP  on the CelebA dataset,
Here, we only consider 27 attributes\footnote{Among the 39 attributes,  12 are heavily skewed  with $ \min(P(Y=1|M),P(Y=1|F))<0.01$ or $\max(P(Y=1|M),P(Y=1|F))>0.99$ (where $M$ represents Male and $F$ represents Female) in the training, validation or test set. They are: 
 ``5 o'Clock Shadow'', ``Bald'', ``Double Chin'', ``Goatee'',  ``Gray Hair'', ``Heavy Makeup'', ``Mustache'', ``No Beard'', ``Rosy Cheeks'', ``Sideburns'', ``Wearing Lipstick'' and ``Wearing Necktie''.} 
 with $0.01\leq P(Y=1|M),P(Y=1|F)\leq 0.99$ in the training, validation, and test sets to ensure that the training, validation and test sample sizes are large enough for each subgroup.
We further identify one attribute, ``Young'', that violates  Condition \ref{suff-con}. We calculate the per-attribute accuracies and DPPs  on the test set. Table \ref{Table-celeba} presents the results of the first six attributes;  the  remaining results are in the Appendix. 
As we can see, even for the large-scale CelebA dataset with high dimensional image features, our algorithm mitigates the gender bias effectively, with almost no loss of accuracy.

\begin{table}
    \caption{Per-attribute accuracy and DPP of the FairBayes-DPP algorithm and unconstrained optimization.}
        \label{Table-celeba}
        \vspace{-0.2cm}
\begin{center}
\setlength{\tabcolsep}{5.0pt}
\renewcommand{\arraystretch}{0.98}
\begin{small}
\begin{sc}
    \begin{tabular}{l|cc|cc}\hline
 \multirow{2}[1]{*}{Attributes}           &      \multicolumn{2}{c|}{Per-attribute Accuracy}&      \multicolumn{2}{c}{Per-attribute DPP}\\\cline{2-5}
                      & FairBayes-DPP& Unconstrained & FairBayes-DPP& Unconstrained \\\hline
 Arched Eyebrows      & 0.838(0.003)   & 0.838(0.003)   & 0.027(0.015)   & 0.099(0.041) \\
 Attractive           & 0.825(0.002)   & 0.826(0.003)   & 0.075(0.011)   & 0.169(0.016) \\
 Bags Under Eyes      & 0.853(0.002)   & 0.852(0.002)   & 0.024(0.015)   & 0.056(0.034) \\
 Bangs                & 0.959(0.001)   & 0.959(0.001)   & 0.007(0.007)   & 0.069(0.029) \\
 Big Lips             & 0.706(0.002)   & 0.717(0.003)   & 0.023(0.015)   & 0.115(0.027) \\
 Big Nose             & 0.845(0.002)   & 0.847(0.003)   & 0.083(0.020)   & 0.145(0.023) \\\hline
  \end{tabular}
  \end{sc}
\end{small}
\end{center}
  \vspace{-0.5cm}
\end{table}

\section{Summary and Discussion}
In this paper, we investigate fair Bayes-optimal classifiers under predictive parity. We prove that when the overall performances of different protected groups vary only moderately, all fair Bayes-optimal classifiers under predictive parity are GWTRs. We further propose a  post-processing algorithm to estimate the optimal GWTR. The derived post-processing algorithm removes the disparity in unconstrained classifiers effectively, while preserving a similar test accuracy.

However, when our sufficient condition is not satisfied, the fair Bayes-optimal classifier under predictive parity may lead to within-group unfairness for the minority group. In the current literature, man algorithms directly apply penalized/constrained optimization to impose fairness. 
Our negative finding, however, is an important reminder  that careful analysis is required
before employing a fairness measure. 
The improper use of a measure may result in severe unintended consequences.

\bibliographystyle{plain} 
\bibliography{fair_optimal}

\begin{thebibliography}{10}

\bibitem{I2020}
Ibrahim Alabdulmohsin.
\newblock Fair classification via unconstrained optimization, 2020.

\bibitem{CFM2009}
Toon Calders, Faisal Kamiran, and Mykola Pechenizkiy.
\newblock Building classifiers with independency constraints.
\newblock In {\em 2009 IEEE International Conference on Data Mining Workshops},
  pages 13--18, 2009.

\bibitem{CFDV2017}
Flavio Calmon, Dennis Wei, Bhanukiran Vinzamuri, Karthikeyan
  Natesan~Ramamurthy, and Kush~R Varshney.
\newblock Optimized pre-processing for discrimination prevention.
\newblock In {\em Advances in Neural Information Processing Systems},
  volume~30. Curran Associates, Inc., 2017.

\bibitem{LV2019}
L.~Elisa Celis and Vijay Keswani.
\newblock Improved adversarial learning for fair classification, 2019.

\bibitem{CHS2020}
Jaewoong Cho, Gyeongjo Hwang, and Changho Suh.
\newblock A fair classifier using kernel density estimation.
\newblock In {\em Advances in Neural Information Processing Systems},
  volume~33, pages 15088--15099, 2020.

\bibitem{C2017fair}
A.~Chouldechova.
\newblock Fair prediction with disparate impact: A study of bias in recidivism
  prediction instruments.
\newblock {\em Big data}, 5(2):153--163, 2017.

\bibitem{CCH02019}
Evgenii Chzhen, Christophe Denis, Mohamed Hebiri, Luca Oneto, and Massimiliano
  Pontil.
\newblock Leveraging labeled and unlabeled data for consistent fair binary
  classification.
\newblock In {\em Advances in Neural Information Processing Systems},
  volume~32. Curran Associates, Inc., 2019.

\bibitem{CSFG2017}
Sam Corbett-Davies, Emma Pierson, Avi Feller, Sharad Goel, and Aziz Huq.
\newblock Algorithmic decision making and the cost of fairness.
\newblock In {\em Proceedings of the 23rd ACM SIGKDD International Conference
  on Knowledge Discovery and Data Mining}, pages 797--806. Association for
  Computing Machinery, 2017.

\bibitem{CJG2019}
A.~Cotter, M.~R. Jiang, H.and~Gupta, S.~Wang, T.~Narayan, S.~You, and
  K.~Sridharan.
\newblock Optimization with non-differentiable constraints with applications to
  fairness, recall, churn, and other goals.
\newblock {\em Journal of Machine Learning Research}, 20(172):1--59, 2019.

\bibitem{CMJW2019}
Elliot Creager, David Madras, Joern-Henrik Jacobsen, Marissa Weis, Kevin
  Swersky, Toniann Pitassi, and Richard Zemel.
\newblock Flexibly fair representation learning by disentanglement.
\newblock In {\em Proceedings of the 36th International Conference on Machine
  Learning}, volume~97 of {\em Proceedings of Machine Learning Research}, pages
  1436--1445. PMLR, 2019.

\bibitem{DDSLKL2009}
Jia Deng, Wei Dong, Richard Socher, Li-Jia Li, Kai Li, and Li~Fei-Fei.
\newblock Imagenet: A large-scale hierarchical image database.
\newblock In {\em 2009 IEEE Conference on Computer Vision and Pattern
  Recognition}, pages 248--255, 2009.

\bibitem{DCL2018BERT}
Jacob Devlin, Ming-Wei Chang, Kenton Lee, and Kristina Toutanova.
\newblock Bert: Pre-training of deep bidirectional transformers for language
  understanding, 2018.

\bibitem{DMB2016}
W.~Dieterich, C.~Mendoza, and T.~Brennan.
\newblock Compas risk scales: Demonstrating accuracy equity and predictive
  parity.
\newblock {\em Northpointe Inc}, 7(4), 2016.

\bibitem{Dua2019}
Dheeru Dua and Casey Graff.
\newblock {UCI} machine learning repository, 2017.

\bibitem{DHPT2012}
Cynthia Dwork, Moritz Hardt, Toniann Pitassi, Omer Reingold, and Richard Zemel.
\newblock Fairness through awareness.
\newblock In {\em Proceedings of the 3rd Innovations in Theoretical Computer
  Science Conference}, ITCS '12, pages 214--226, 2012.

\bibitem{C2001CS}
Charles Elkan.
\newblock The foundations of cost-sensitive learning.
\newblock In {\em In Proceedings of the Seventeenth International Joint
  Conference on Artificial Intelligence}, pages 973--978, 2001.

\bibitem{FFMS2015}
Michael Feldman, Sorelle~A. Friedler, John Moeller, Carlos Scheidegger, and
  Suresh Venkatasubramanian.
\newblock Certifying and removing disparate impact.
\newblock In {\em Proceedings of the 21th ACM SIGKDD International Conference
  on Knowledge Discovery and Data Mining}, pages 259--268. Association for
  Computing Machinery, 2015.

\bibitem{BJA2016}
Benjamin Fish, Jeremy Kun, and {\'{A}}d{\'{a}}m~D{\'{a}}niel Lelkes.
\newblock A confidence-based approach for balancing fairness and accuracy.
\newblock In {\em Proceedings of the 2016 {SIAM} International Conference on
  Data Mining, Miami, Florida, USA, May 5-7, 2016}, pages 144--152. {SIAM},
  2016.

\bibitem{FBC2016}
A.~W. Flores, K.~Bechtel, and C.~T. Lowenkamp.
\newblock False positives, false negatives, and false analyses: A rejoinder to
  machine bias: There's software used across the country to predict future
  criminals. and it's biased against blacks.
\newblock {\em Fed. Probation}, 80:38, 2016.

\bibitem{GCAM2016}
Gabriel Goh, Andrew Cotter, Maya Gupta, and Michael~P Friedlander.
\newblock Satisfying real-world goals with dataset constraints.
\newblock In {\em Advances in Neural Information Processing Systems},
  volume~29. Curran Associates, Inc., 2016.

\bibitem{MQ2017}
Megh Gupta and Qasim Mohammad.
\newblock Advances in ai and ml are reshaping healthcare, 2017.

\bibitem{HPS2016}
Moritz Hardt, , Eric Price, and Nati Srebro.
\newblock Equality of opportunity in supervised learning.
\newblock In {\em Advances in Neural Information Processing Systems},
  volume~29, 2016.

\bibitem{HZRS15Res}
Kaiming He, Xiangyu Zhang, Shaoqing Ren, and Jian Sun.
\newblock Deep residual learning for image recognition, 2015.

\bibitem{JL2019}
J.~E. Johndrow and K.~Lum.
\newblock An algorithm for removing sensitive information: application to
  race-independent recidivism prediction.
\newblock {\em The Annals of Applied Statistics}, 13(1):189--220, 2019.

\bibitem{JKMR2016}
Matthew Joseph, Michael Kearns, Jamie~H Morgenstern, and Aaron Roth.
\newblock Fairness in learning: Classic and contextual bandits.
\newblock In {\em Advances in Neural Information Processing Systems},
  volume~29. Curran Associates, Inc., 2016.

\bibitem{JJSL2016}
Surya~Mattu Julia~Angwin, Jeff~Larson and Lauren Kirchner.
\newblock Machine bias there's software used across the country to predict
  future criminals. and it's biased against blacks, 2016.

\bibitem{KC2012}
F.~Kamiran and T.~Calders.
\newblock Data preprocessing techniques for classification without
  discrimination.
\newblock {\em Knowledge and Information Systems}, 33(1):1--33, 2012.

\bibitem{PKG2019}
Preethi Lahoti, Krishna~P. Gummadi, and Gerhard Weikum.
\newblock ifair: Learning individually fair data representations for
  algorithmic decision making.
\newblock In {\em 35th {IEEE} International Conference on Data Engineering,
  {ICDE} 2019, Macao, China, April 8-11, 2019}, pages 1334--1345. {IEEE}, 2019.

\bibitem{LBRSP2021}
Joshua~K Lee, Yuheng Bu, Deepta Rajan, Prasanna Sattigeri, Rameswar Panda,
  Subhro Das, and Gregory~W Wornell.
\newblock Fair selective classification via sufficiency.
\newblock In {\em Proceedings of the 38th International Conference on Machine
  Learning}, volume 139 of {\em Proceedings of Machine Learning Research},
  pages 6076--6086. PMLR, 18--24 Jul 2021.

\bibitem{LSMH2019}
Lydia~T. Liu, Max Simchowitz, and Moritz Hardt.
\newblock The implicit fairness criterion of unconstrained learning.
\newblock In {\em Proceedings of the 36th International Conference on Machine
  Learning}, volume~97 of {\em Proceedings of Machine Learning Research}, pages
  4051--4060. PMLR, 09--15 Jun 2019.

\bibitem{CelebAdata}
Ziwei Liu, Ping Luo, Xiaogang Wang, and Xiaoou Tang.
\newblock Deep learning face attributes in the wild.
\newblock In {\em 2015 IEEE International Conference on Computer Vision
  (ICCV)}, pages 3730--3738, 2015.

\bibitem{CKYMR2016}
Christos Louizos, Kevin Swersky, Yujia Li, Max Welling, and Richard~S. Zemel.
\newblock The variational fair autoencoder.
\newblock In {\em 4th International Conference on Learning Representations,
  {ICLR} 2016, San Juan, Puerto Rico, May 2-4, 2016, Conference Track
  Proceedings}, 2016.

\bibitem{KJ2016}
Kristian Lum and James Johndrow.
\newblock A statistical framework for fair predictive algorithms, 2016.

\bibitem{DETR2018}
David Madras, Elliot Creager, Toniann Pitassi, and Richard Zemel.
\newblock Learning adversarially fair and transferable representations.
\newblock In {\em Proceedings of the 35th International Conference on Machine
  Learning}, volume~80 of {\em Proceedings of Machine Learning Research}, pages
  3384--3393. PMLR, 10--15 Jul 2018.

\bibitem{MW2018}
Aditya~Krishna Menon and Robert~C Williamson.
\newblock The cost of fairness in binary classification.
\newblock In {\em Proceedings of the 1st Conference on Fairness, Accountability
  and Transparency}, volume~81 of {\em Proceedings of Machine Learning
  Research}, pages 107--118. PMLR, 23--24 Feb 2018.

\bibitem{miller1993loss}
John~W Miller, Rod Goodman, and Padhraic Smyth.
\newblock On loss functions which minimize to conditional expected values and
  posterior probabilities.
\newblock {\em IEEE Transactions on Information Theory}, 39(4):1404--1408,
  1993.

\bibitem{NH2018}
Harikrishna Narasimhan.
\newblock Learning with complex loss functions and constraints.
\newblock In {\em Proceedings of the Twenty-First International Conference on
  Artificial Intelligence and Statistics}, volume~84 of {\em Proceedings of
  Machine Learning Research}, pages 1646--1654. PMLR, 2018.

\bibitem{PMK2017}
Geoff Pleiss, Manish Raghavan, Felix Wu, Jon Kleinberg, and Kilian~Q
  Weinberger.
\newblock On fairness and calibration.
\newblock In {\em Advances in Neural Information Processing Systems},
  volume~30. Curran Associates, Inc., 2017.

\bibitem{RKSR2021}
Vikram~V. Ramaswamy, Sunnie S.~Y. Kim, and Olga Russakovsky.
\newblock Fair attribute classification through latent space de-biasing.
\newblock In {\em 2021 IEEE/CVF Conference on Computer Vision and Pattern
  Recognition (CVPR)}, pages 9297--9306, 2021.

\bibitem{RBMF2020}
Anian Ruoss, Mislav Balunovic, Marc Fischer, and Martin Vechev.
\newblock Learning certified individually fair representations.
\newblock In {\em Advances in Neural Information Processing Systems 33}, 2020.

\bibitem{SHC2019}
P.~Sattigeri, S.~C. Hoffman, V.~Chenthamarakshan, and K.~R. Varshney.
\newblock Fairness gan: Generating datasets with fairness properties using a
  generative adversarial network.
\newblock {\em IBM Journal of Research and Development}, 63(4/5):3:1--3:9,
  2019.

\bibitem{SC2021}
Nicolas Schreuder and Evgenii Chzhen.
\newblock Classification with abstention but without disparities.
\newblock In {\em Proceedings of the Thirty-Seventh Conference on Uncertainty
  in Artificial Intelligence}, volume 161 of {\em Proceedings of Machine
  Learning Research}, pages 1227--1236. PMLR, 27--30 Jul 2021.

\bibitem{KA2014VGG}
Karen Simonyan and Andrew Zisserman.
\newblock Very deep convolutional networks for large-scale image recognition,
  2014.

\bibitem{SVQ2014seq2}
Ilya Sutskever, Oriol Vinyals, and Quoc~V Le.
\newblock Sequence to sequence learning with neural networks.
\newblock In {\em Advances in Neural Information Processing Systems},
  volume~27. Curran Associates, Inc., 2014.

\bibitem{SLJS2015GOOGLE}
Christian Szegedy, Wei Liu, Yangqing Jia, Pierre Sermanet, Scott Reed, Dragomir
  Anguelov, Dumitru Erhan, Vincent Vanhoucke, and Andrew Rabinovich.
\newblock Going deeper with convolutions.
\newblock In {\em Proceedings of the IEEE Conference on Computer Vision and
  Pattern Recognition (CVPR)}, June 2015.

\bibitem{SVI2016INC}
Christian Szegedy, Vincent Vanhoucke, Sergey Ioffe, Jon Shlens, and Zbigniew
  Wojna.
\newblock Rethinking the inception architecture for computer vision.
\newblock In {\em Proceedings of the IEEE Conference on Computer Vision and
  Pattern Recognition (CVPR)}, June 2016.

\bibitem{SEC2019}
Song{\"{u}}l Tolan, Marius Miron, Emilia G{\'{o}}mez, and Carlos Castillo.
\newblock Why machine learning may lead to unfairness: Evidence from risk
  assessment for juvenile justice in catalonia.
\newblock In {\em Proceedings of the Seventeenth International Conference on
  Artificial Intelligence and Law, {ICAIL} 2019, Montreal, QC, Canada, June
  17-21, 2019}, pages 83--92. {ACM}, 2019.

\bibitem{VSPU2017Attention}
Ashish Vaswani, Noam Shazeer, Niki Parmar, Jakob Uszkoreit, Llion Jones,
  Aidan~N Gomez, \L~ukasz Kaiser, and Illia Polosukhin.
\newblock Attention is all you need.
\newblock In {\em Advances in Neural Information Processing Systems},
  volume~30. Curran Associates, Inc., 2017.

\bibitem{CFC2018}
Christina Wadsworth, Francesca Vera, and Chris Piech.
\newblock Achieving fairness through adversarial learning: an application to
  recidivism prediction, 2018.

\bibitem{WQK2020}
Zeyu Wang, Klint Qinami, Ioannis Karakozis, Kyle Genova, Prem Nair, Kenji Hata,
  and Olga Russakovsky.
\newblock Towards fairness in visual recognition: Effective strategies for bias
  mitigation.
\newblock In {\em IEEE/CVF Conference on Computer Vision and Pattern
  Recognition (CVPR)}, 2020.

\bibitem{XWYZW2019}
Depeng Xu, Yongkai Wu, Shuhan Yuan, Lu~Zhang, and Xintao Wu.
\newblock Achieving causal fairness through generative adversarial networks.
\newblock In {\em Proceedings of the Twenty-Eighth International Joint
  Conference on Artificial Intelligence, {IJCAI-19}}, pages 1452--1458.
  International Joint Conferences on Artificial Intelligence Organization,
  2019.

\bibitem{XYZ2018}
Depeng Xu, Shuhan Yuan, Lu~Zhang, and Xintao Wu.
\newblock Fairgan: Fairness-aware generative adversarial networks.
\newblock In {\em 2018 IEEE International Conference on Big Data (Big Data)},
  pages 570--575, 2018.

\bibitem{XYfairplus}
Depeng Xu, Shuhan Yuan, Lu~Zhang, and Xintao Wu.
\newblock Fairgan<sup>+</sup>: Achieving fair data generation and
  classification through generative adversarial nets.
\newblock In {\em 2019 IEEE International Conference on Big Data (Big Data)},
  pages 1401--1406, 2019.

\bibitem{YDY2019XLNET}
Zhilin Yang, Zihang Dai, Yiming Yang, Jaime Carbonell, Russ~R Salakhutdinov,
  and Quoc~V Le.
\newblock Xlnet: Generalized autoregressive pretraining for language
  understanding.
\newblock In {\em Advances in Neural Information Processing Systems},
  volume~32. Curran Associates, Inc., 2019.

\bibitem{ZVGGK2017}
Muhammad~Bilal Zafar, Isabel Valera, Manuel Gomez~Rodriguez, and Krishna~P.
  Gummadi.
\newblock Fairness beyond disparate treatment and disparate impact: Learning
  classification without disparate mistreatment.
\newblock In {\em Proceedings of the 26th International Conference on World
  Wide Web}, pages 1171--1180. International World Wide Web Conferences
  Steering Committee, 2017.

\bibitem{ZWKT2013}
Rich Zemel, Yu~Wu, Kevin Swersky, Toni Pitassi, and Cynthia Dwork.
\newblock Learning fair representations.
\newblock In {\em Proceedings of the 30th International Conference on Machine
  Learning}, volume~28 of {\em Proceedings of Machine Learning Research}, pages
  325--333. PMLR, 2013.

\bibitem{ZDC2022}
Xianli Zeng, Edgar Dobriban, and Guang Cheng.
\newblock Bayes-optimal classifiers under group fairness, 2022.

\bibitem{ZLM2018}
Brian~Hu Zhang, Blake Lemoine, and Margaret Mitchell.
\newblock Mitigating unwanted biases with adversarial learning.
\newblock In {\em Proceedings of the 2018 AAAI/ACM Conference on AI, Ethics,
  and Society}, AIES '18, pages 335--340. Association for Computing Machinery,
  2018.

\bibitem{ZWY2017}
Jieyu Zhao, Tianlu Wang, Mark Yatskar, Vicente Ordonez, and Kai-Wei Chang.
\newblock Men also like shopping: Reducing gender bias amplification using
  corpus-level constraints.
\newblock In {\em Proceedings of the 2017 Conference on Empirical Methods in
  Natural Language Processing}, pages 2941--2951, 2017.

\end{thebibliography}



\newpage

\appendix

 \section{Proof of Theorem \ref{thm-GWTR}}

{\bf Additional notations.}
We use the following notations.
A Bernoulli random variable $B$ with success probability $p\in[0,1]$ is denoted as $B\sim\be(p)$. We denote $p_a:=P(A=a)$ and $p_{Y|a}:=P(Y=1|A=a)$; 
Further, we denote by $P_X$ and $P_{X|a}$  the distribution of $X$ and the conditional distribution of $X$ given $A=a$, respectively.

The proof of Theorem \ref{thm-GWTR} relies on the following two technical lemmas.
For any $\tau\in[0,1]$, consider a Bernoulli random variable $B\sim \be(\tau)$, independent of other sources of randomness considered. 
For all $a\in\mathcal{A}$, $t\in[0,1]$ and $\tau \in [0,1]$, 
define the random variable $\hat{Y}_{a,t,\tau}$ by
$$\hat{Y}_{a,t,\tau}=I(\eta_a(X)>t)+B\cdot I(\eta_a(X)=t).$$
For all $a\in\mathcal{A}$, define the set $S$
$$
S = \{(t,\tau)\in[0,1]^2:\,P(\eta_A(X)>t|A=a)
+\tau\cdot P(\eta_A(X)=t|A=a)>0\}.
$$
For all $(t,\tau)\in S$,
denote
\begin{equation}\label{ga}
    g_a(t,\tau)=P(Y=1|\hat{Y}_{a,t,\tau}=1,A=a).
\end{equation}
This is well-defined due to the definition of $S$.

\begin{lemma}\label{lem1}
For all $a\in\mathcal{A}$, $t\in[0,1]$ 
and $0\leq \tau_1\leq \tau_2\leq 1$, 
such that $(t,\tau_1)\in S$,
we have
\begin{equation}\label{monotau}
   g_a(t,\tau_1)\geq g_a(t,\tau_2)\geq t. 
\end{equation}
Furthermore, 
for all $a\in\mathcal{A}$, $ \tau_1,\tau_2\in[0,1]$ and $0\leq t_1\leq t_2\leq 1$
such that $(t_1,\tau_1)\in S$, $(t_1,\tau_2)\in S$,
\begin{equation}\label{monot}
g_a(t_1,\tau_1)\leq g_a(t_2,\tau_2).
\end{equation}
\end{lemma}

\begin{proof}
For all $a\in\A$, $y\in\Y$, and $t\in [0,1]$,
denote 
\begin{align*}
w_{ay}(t)&=P(\eta_A(X)>t|A=a,Y=y), \qquad v_{ay}(t)=P(\eta_A(X)=t|A=a,Y=y), \\
w_{a}(t)&=P(\eta_A(X)>t|A=a), \qquad
v_{a}(t)=P(\eta_A(X)=t|A=a).
\end{align*}
Let $0\leq t_1\leq t_2\leq 1$. 
Recalling the conditional density $dP_{X|a,y}$ of $X$ given $A=a$ and $Y=y$, 
we have that $\eta_a(x)=\frac{p_{Y,a}dP_{X|A=a,Y=1}(x)}{dP_{X|a}(x)}.$ 
We thus have for all $t\in [0,1]$ for which $w_{a}(t)>0$ that
\begin{eqnarray*}
\frac{p_{Y|a}w_{a1}(t)}{w_{a}(t)}&=&\frac{p_{Y|a}\int_{\eta_a(x)>t}dP_{X|A=a,Y=1}(x)}{\int_{\eta_a(x)>t}dP_{X|a}(x)}
=\frac{\int_{\eta_a(x)>t}\eta_a(x)dP_{X|a}(x)}{\int_{\eta_a(x)>t}dP_{X|a}(x)}>t.
\end{eqnarray*}
Further, when $v_{a}(t)>0$,
\begin{eqnarray*}
\frac{p_{Y|a}v_{a1}(t)}{v_{a}(t)}&=&\frac{p_{Y|a}\int_{\eta_a(x)=t}dP_{X|A=a,Y=1}(x)}{\int_{\eta_a(x)=t}dP_{X|a}(x)}
=\frac{\int_{\eta_a(x)=t}\eta_a(x)dP_{X|a}(x)}{\int_{\eta_a(x)=t}dP_{X|a}(x)}=t.
\end{eqnarray*}
It follows that, 
for $t\in [0,1]$ 
and $0\leq \tau_1\leq\tau_2\leq 1$,
such that $(t,\tau_1)\in S$,
\begin{eqnarray*}
t\leq \frac{w_{a1}(t)+ v_{a1}(t)}{w_{a}(t)+v_{a}(t)}\leq \frac{w_{a1}(t)+\tau_2 v_{a1}(t)}{w_{a}(t)+\tau_2 v_{a}(t)}\leq\frac{w_{a1}(t)+\tau_1 v_{a1}(t)}{w_{a}(t)+\tau_1 v_{a}(t)}.
\end{eqnarray*}
Eq. (\ref{monotau}) follows since for all $t,\tau\in[0,1]$ such that $(t,\tau)\in S$,
$$g_a(t,\tau)=\frac{p_{Y,a}[w_{a1}(t)+\tau v_{a1}(t)]}{w_{a}(t)+\tau v_{a}(t)}.$$

For Eq. (\ref{monot}), we have that, when $0\leq t_1\leq t_2\leq 1$
and  $P(\eta_A(X)>t_2|A=a)>0$,
\begin{eqnarray*}
g_a(t_{1},\tau_{1})-g_a(t_{2},\tau_{2})
&=&\frac{p_{Y|a}[w_{a1}(t_1)+\tau_1 v_{a1}(t_1)]}{w_{a}(t_1)+\tau_1v_{a}(t_1)}-\frac{p_{Y,a}[w_{a1}(t_2)+\tau_2 v_{a1}(t_2)]}{w_{a}(t_2)+\tau_2v_{a}(t_2)}\\
&\leq&\frac{p_{Y,a}w_{a1}(t_1)}{w_{a}(t_1)}-\frac{p_{Y,a}[w_{a1}(t_2)+ v_{a1}(t_2)]}{w_{a}(t_2)+v_{a}(t_2)}.
\end{eqnarray*}
This further equals
\begin{eqnarray*}
&&\frac{\int_{\eta_a(x)>t_1}\eta_a(x)dP_{X|a}(x)}{\int_{\eta_a(x)>t_1}dP_{X|a}(x)}-\frac{\int_{\eta_a(x)\geq t_2}\eta_a(x)dP_{X|a}(x)}{\int_{\eta_a(x)\geq t_2}dP_{X|a}(x)}\\
&=&\frac{\int_{t_1<\eta_a(x)<t_2}\eta_a(x)dP_{X|a}(x)+\int_{\eta_a(x)\geq t_2}\eta_a(x)dP_{X|a}(x)}{\int_{t_1<\eta_a(x)< t_2}dP_{X|a}(x)+\int_{\eta_a(x)\geq t_2}dP_{X|a}(x)}
-\frac{\int_{\eta_a(x)\geq t_2}\eta_a(x)dP_{X|a}(x)}{\int_{\eta_a(x)\geq t_2}dP_{X|a}(x)}.
\end{eqnarray*}
This can also be written as
\begin{eqnarray*}
&&\frac{\int_{t_1<\eta_a(x)< t_2}\eta_a(x)dP_{X|a}(x)\cdot\int_{\eta_a(x)\geq t_2}dP_{X|a}(x)}{[\int_{t_1<\eta_a(x)< t_2}dP_{X|a}(x)+\int_{\eta_a(x)\geq t_2}dP_{X|a}(x)]\cdot\int_{\eta_a(x)\geq t_2}dP_{X|a}(x)}\\
&&-\frac{\int_{t_1<\eta_a(x)< t_2}dP_{X|a}(x)\cdot\int_{\eta_a(x)\geq t_2}\eta_a(x)dP_{X|a}(x)}{[\int_{t_1<\eta_a(x)< t_2}dP_{X|a}(x)+\int_{\eta_a(x)\geq t_2}dP_{X|a}(x)]\cdot\int_{\eta_a(x)\geq t_2}dP_{X|a}(x)}
\leq 0.
\end{eqnarray*}
This finishes the proof.
\end{proof}

\begin{lemma}\label{lem2}
For any $a\in\A$
and $s\in[p_{Y|a},1]$, 
there exists $(t_s,\tau_s)\in[0,1]^2$ such that, 
with $g_a$ from \eqref{ga}, 
$$g_a(t_s,\tau_s)=s.$$

\end{lemma}
\begin{proof}
For all $a\in\mathcal{A}$, define the sets $T_0$ on which $g_a(t,0)$ and $g_a(t,1)$, respectively, are well-defined:
\begin{align*}
T_0 &= \{t\in[0,1]:\,P(\eta_A(X)>t|A=a)>0\}\\
T_1 &= \{t\in[0,1]:\,P(\eta_A(X)\geq t|A=a)>0\}.
\end{align*}

As a function of 
$t\in T_1$,
$t\mapsto g_a(t,1)$  is left-continuous.
Letting $t^* = \sup T_0 \in [0,1]$, we have $T_0 = [0,t^*)$ and $T_1 = [0,t^*]$.

Now, since
$g_a(0,1) = p_{Y|a} \leq s$,
${t}_s=\sup\{t \in T_1:g_a(t,1)\leq s\}$ is
well-defined.  
From Lemma \ref{lem1}, 
the definition of $t_s$, and the left-continuity of $t\mapsto g_a(t,1)$ on $T_1$,
it follows that
$$g_a(t_s,1)\leq s\leq g_a(t_s,0).$$
(1) When $P(\eta_a(X)=t_s|A=1)=0$, for all $\tau\in[0,1]$ we have
$$g_a(t_s,0)=g_a(t_s,\tau)=g_a(t_s,1)=s.$$
In this case, we can set  $\tau_s\in[0,1]$.

(2) When $P(\eta_a(X)>t_s|A=1)=0$ for $a\in\mathcal{A}$,
we have $s=t_s$ and we can set  $\tau_s\in[0,1]$.

(3) When $P(\eta_a(X)=t_s|A=1)\neq 0$, we have  $g_a(t_s,\tau_s)=s$ for
\begin{equation*}
\tau_s=\frac{p_{Y|a}\cdot P(\eta_a(X)>t_s|A=a,Y=1)-s\cdot P(\eta_a(X)>t_s|A=a)}{p_{Y|a}\cdot P(\eta_a(X)=t_s|A=a)-s\cdot P(\eta_a(X)=t_s|A=a,Y=1)}.    
\end{equation*}
\end{proof}

\begin{lemma}\label{lem3}
Let  $f$ be any classifier and $f_G=I(\eta_a(x)>t_a)+\tau_a(x)I(\eta_a(x)=t_a)$ be a GWTR satisfies 
\begin{equation}\label{condition_tau}
I(\tau_a(x)\equiv 1)+ 
I\left({\int f_G(x,a)\eta_a(x)dP_{X|a}(x)}>t_a{\int f_G(x,a)dP_{X|a}(x)}\right)\geq 1.
\end{equation}
Suppose that, for all $a\in\mathcal{A}$,
\begin{equation}\label{eq:lem31}
{\int f(x,a)dP_{X|a}(x)}={\int f_G(x,a)dP_{X|a}(x)}\end{equation}
and
\begin{equation}\label{eq:lem32}\frac{\int f(x,a)\eta_a(x)dP_{X|a}(x)}{\int f(x,a)dP_{X|a}(x)}=\frac{\int f_G(x,a)\eta_a(x)dP_{X|a}(x)}{\int f_G(x,a)dP_{X|a}(x)}.\end{equation}
Then, $f$ is also a GWTR. 
Conversely, if  $f$ is not a GWTR and \eqref{eq:lem32} holds
for all $a\in\mathcal{A}$, we have
$$\sum_{a=1}^{|\mathcal{A}|}p_a{\int [f_G(x,a)-f(x,a)]dP_{X|a}(x)}>0.$$
\end{lemma}
\begin{proof}
We assume $f_G$ takes the following form: for all $x\in\X$ and $a\in\A$,
\begin{equation*}
    f_G(x,a)=I(\eta_a(x)>t_a)+\tau_a(x,a) I(\eta_a(x)=t^\dag_a).
\end{equation*}
From \eqref{eq:lem31} and \eqref{eq:lem32}, we have
\begin{eqnarray}\label{lem3:eq1}
\nonumber&&\int (f(x,a)-f_G(x,a))dP_{X|a}(x)=\int_{\eta(x)>t_a}(f(x,a)-1)dP_{X|a}(x)\\
&&\qquad+\int_{\eta(x)<t_a}f(x,a)dP_{X|a}(x)+\int_{\eta(x)=t_a}(f(x,a)-\tau_a(x))dP_{X|a}(x)=0,
\end{eqnarray}
and
\begin{eqnarray}\label{lem3:eq2}
\nonumber&&\int (f(x,a)-f_G(x,a))\eta_a(x)dP_{X|a}(x)=\int_{\eta(x)>t_a}(f(x,a)-1)\eta_a(x)dP_{X|a}(x)\\
&&\qquad+\int_{\eta(x)<t_a}f(x,a)dP_{X|a}(x)+t_a\int_{\eta(x)=t_a}(f(x,a)-\tau_a(x))dP_{X|a}(x)=0.
\end{eqnarray}
Combining \eqref{lem3:eq1} and \eqref{lem3:eq2} gives us, for all $a\in\mathcal{A}$,
\begin{equation*}\int_{\eta_a(x)>t_a}(f(x,a)-1)(\eta_a(x)-t_a)dP_{X|a}(x)+\int_{\eta_a(x)<t_a}f(x,a)(\eta_a(x)-t_a)dP_{X|a}(x)=0.
\end{equation*}
Noting that $I(\eta_a(x)>t_a)(f(x,a)-1)(\eta_a(x)-t_a)\leq 0$  and $I(\eta_a(x)<t_a)f(x,a)(\eta_a(x)-t_a)\leq 0$,  we have
\begin{equation}\label{lem3:eq3}\int_{\eta_a(x)>t_a}(f(x,a)-1)(\eta_a(x)-t_a)dP_{X|a}(x)+\int_{\eta_a(x)<t_a}f(x,a)(\eta_a(x)-t_a)dP_{X|a}(x)\leq 0.
\end{equation}
The equality holds if and only if,  for all $a\in\mathcal{A}$,
$f(x,a)=f_G(x,a)$ almost surely on the set
$\{\eta_a(x)>t_a\}\cup \{\eta_a(x)>t_a\}$. In other words, $f$ is also a GWTR.

When $f$ is not a GWTR, let $$\frac{\int f(x,a)\eta_a(x)dP_{X|a}(x)}{\int f(x,a)dP_{X|a}(x)}=\frac{\int f_G(x,a)\eta_a(x)dP_{X|a}(x)}{\int f_G(x,a)dP_{X|a}(x)}=s_G.$$
We have $0\leq t_a\leq s_G$ by Lemma \ref{lem2}. Suppose there exists a  $a\in\mathcal{A}$ such that
\begin{equation}\label{contradiction}
{\int f(x,a)dP_{X|a}(x)}>{\int f_G(x,a)dP_{X|a}(x)}.
\end{equation}
(1) When $t_a<s_G$, we have,
\begin{eqnarray*}
&&{\int [f(x,a)-f_G(x,a)]\eta_a(x)dP_{X|a}(x)}- t_a{\int [f(x,a)-f_G(x,a)]dP_{X|a}(x)}\\
&=&\int_{\eta_a(x)>t_a} (f(x,a)-1)(\eta_a(x)-t_a)dP_{X|a}(x)+\int_{\eta_a(x)<t_a} f(x,a)(\eta_a(x)-t_a)dP_{X|a}(x)>0.
\end{eqnarray*}
This contradicts \eqref{lem3:eq3}.

(2) When $t_a=s_G$, we have $f(x,a)=I(\eta(x,a)\geq t_a)$. Then,
\begin{equation*}
\int_{\eta_a(x)>t_a} (f(x,a)-1)(\eta_a(x)-t_a)dP_{X|a}(x)+\int_{\eta_a(x)<t_a} f(x,a)(\eta_a(x)-t_a)dP_{X|a}(x)
=0.
\end{equation*}
This equation  holds if and only if 
$f(x,a)=f_G(x,a)$ almost surely on the set
$\{\eta_a(x)>t_a\}\cup \{\eta_a(x)>t_a\}$. Then,
\begin{equation*}
{\int f(x,a)dP_{X|a}(x)}-{\int f_G(x,a)dP_{X|a}(x)}={\int_{\eta(x,a)=t_a} (f(x,a)-1)dP_{X|a}(x)}\leq 0.
\end{equation*}
Again, we have a contradiction since  $\int f(x,a)dP_{X|a}(x)-\int f_G(x,a)dP_{X|a}(x)>0$.

As a result, we can conclude that, for all $a\in\mathcal{A}$,  $$\int f(x,a)dP_{X|a}(x)\leq \int f_G(x,a)dP_{X|a}(x).$$
Moreover, there exists at least one  $a\in\mathcal{A}$ such that
 $$\int f(x,a)dP_{X|a}(x)< \int f_G(x,a)dP_{X|a}(x).$$
 Otherwise, $f$ is also a GWTR. This finishes the proof.
\end{proof}

We adopt the following strategy to prove Theorem \ref{thm-GWTR}. 
Consider any classifier $f$ that satisfies predictive parity, which is not a GWTR. 
We will show that there exist a GWTR
satisfying predictive parity with a smaller risk. 
Thus, at least one of the fair Bayes-optimal classifier under predictive parity is a GWTR.

Recall that $\hat{Y}_f$ is the prediction of $f$ at  $(x,a)$. 
As $f$ satisfies predictive parity, there exists  $s_f\in[0,1]$ such that
$$P(Y=1|A=a,\hat{Y}_f=1)=s_f\leq 1 \ \ \ \text{for } a\in\mathcal{A}.$$

We set
\begin{equation}
s^\dag =\left\{ \begin{array}{lcc}
        \max\,(s_f,\max_a\, p_{Y|a}), && \max\,(s_f,\max_a\, p_{Y|a})>c; \\
      c+\ep,   &&  \max\,(s_f,\max_a\, p_{Y|a})\leq c.
    \end{array}\right.
\end{equation}
Here, $\ep<1-c$ is a small constant such that there exists a $a\in\mathcal{A}$ with
$P(\eta_a(X)>c+\ep|A=a)>0$.
By our construction, we have $s^\dag\in[\max_a p_{Y|a},1]$ and, according to Lemma \ref{lem2}, there exist combinations $(t^\dag_a,\tau^\dag_a)_{a=1}^{|\mathcal{A}|}$ such that, for $g_a$ from \eqref{ga}, 
\begin{equation}\label{t_tau_a}
    g_a(t^\dag_{a},\tau^\dag_{a})=s^\dag, \ \ \ \ a\in\mathcal{A}.
\end{equation}

Now, we consider the GWTR $f^\dag$ defined for all $x\in\X$ and $a\in\A$ by
\begin{equation}\label{fd}
    f^\dag(x,a)=I(\eta_a(x)>t^\dag_a)+\tau^\dag_a I(\eta_a(x)=t^\dag_a).
\end{equation}
Here, we  follow the construction in Lemma \ref{lem2} to set
$t_a^\dag=\sup \{t: g_a(t)<s^\dag\}$, and let $\tau_a(x)\equiv \tau_a^\dag$ be a constant function. Moreover, we set $\tau_a^\dag= 1$ whenever $P(\eta_a(X)>t_a^\dag|A=a)=0$ or  $P(\eta_a(X)=t_a^\dag|A=a)=0$.
Clearly, $f^\dag$ satisfies predictive parity, 
and thus it is enough to show that $f^\dag$ has a smaller risk  than $f$, i.e., $R_c(f^\dag)-R_c(f)< 0$.
Now, we can write
\begin{eqnarray*}
R_c(f)
&=&\sum_{a\in\mathcal{A}}\left[(1-c)P(\hat{Y}_f=0,Y=1,A=a)+c\cdot P(\hat{Y}_f=1,Y=0,A=a)\right]\\
&=&
(1-c)P(Y=1)
- \sum_{a\in\mathcal{A}}p_a(1-c)\int f(x,a)\eta_a(x)dP_{X|a}(x)\\
&&+\sum_{a\in\mathcal{A}}p_ac\int f(x,a)(1-\eta_a(x))dP_{X|a}(x).
\end{eqnarray*}
Next, for any classifier $f$ satisfying predictive parity with positive predictive value $s_f$, we have that
\begin{eqnarray*}
s_f&=P(Y=1|\hat{Y}_f=1,A=a)=\frac{p_{Y|a}P(\hat{Y}_f=1|Y=1,A=a)}{P(\hat{Y}_f=1|A=a)}\\
&=\frac{p_{Y|a}\int f(x,a)dP_{X|A=a,Y=1}(x)}{\int f(x,a)dP_{X|a}(x)}
=\frac{\int f(x,a)\eta_a(x)dP_{X|a}(x)}{\int f(x,a)dP_{X|a}(x)}.
\end{eqnarray*}
It follows that $R_c(f)$ further equals
\begin{eqnarray*}
&&\sum_{a\in\mathcal{A}}p_a\int f(x,a)(c-\eta_a(x))dP_{X|a}(x)+(1-c)P(Y=1)\\
&=&\sum_{a\in\mathcal{A}}p_a(c-s_f)\int f(x,a)dP_{X|a}(x)+(1-c)P(Y=1).
\end{eqnarray*}
As a result, $R_c(f^\dag)-R_c(f)$ equals
\begin{eqnarray*}
\sum_{a\in\mathcal{A}}p_a(c-s^\dag)\int f^\dag(x,a)dP_{X|a}(x)
 -\sum_{a\in\mathcal{A}}p_a(c-s_f)\int f(x,a)dP_{X|a}(x).
\end{eqnarray*}

We consider the following three cases in order: (1) $s_f\leq \min(c,\max_ap_{Y|a})$, (2) $s_f>  \max(c,\max_ap_{Y|a})$, and (3) $\min(c,\max_ap_{Y|a})< s_f \leq \max(c,\max_ap_{Y|a})$.

~

(1) Case 1: $s_f\leq \min(c,\max_ap_{Y|a})$.

It is clear that $ R_c(f^\dag)-R_c(f)< 0$ since $c - s^\dag< 0$ and $c-s_f\geq 0$.

~

(2) Case 2: $s_f>  \max(c,\max_ap_{Y|a})$. 

We have from the definition of $s^\dag$, \eqref{t_tau_a} and \eqref{monotau} that
for all $a\in\A$,
$s^\dag=s_f\geq t^\dag_a$. 
Further, we can write 
$$R_c(f^\dag)-R_c(f)= 
\sum_{a\in\mathcal{A}}p_a(c-s_f)\int [f^\dag(x,a)-f(x,a)]dP_{X|a}(x).$$
Suppose that $s_f=t_a$. 
Specifically, $s_f=s^\dag=t_a^\dag$ equals
\begin{eqnarray*}
t_a^\dag=\frac{\int f^\dag(x,a)\eta_a(x)dP_{X|a}(x)}{\int f^\dag(x,a)dP_{X|a}(x)}
&=&\frac{\int_{\eta_a(x)>t_a^\dag}\eta_a(x)dP_{X|a}(x)+\tau^\dag_a\int_{\eta_a(x)=t_a^\dag}\eta_a(x)dP_{X|a}(x)}{\int_{\eta_a(x)>t_a^\dag}dP_{X|a}(x)+\tau^\dag_a\int_{\eta_a(x)=t_a^\dag}dP_{X|a}(x)}.
\end{eqnarray*}
This implies $P(\eta_a(X)>t_a^\dag|A=a)=0$ and, by our construction, $\tau_a(x)\equiv \tau^\dag_a= 1$. Thus, $f^\dag$ satisfies the condition \eqref{condition_tau} in Lemma \ref{lem3}. As a result, we have,
$$\sum_{a=1}^{|\mathcal{A}|}p_a{\int [f^\dag(x,a)-f(x,a)]dP_{X|a}(x)}>0.$$
This implies $ R_c(f^\dag)-R_c(f)< 0$ since $c-s_f<0$.

~

(3) Case 3: $\min(c,\max_ap_{Y|a})<  s_f \leq \max(c,\max_ap_{Y|a})$. 

In fact, the case 3 can be further divided into  two possible sub-cases, depending on the relations between  $c$ and  $\max_ap_{Y|a}$: (3.i) $\max_a p_{Y|a}<s_f \leq c$, and (3.ii) $c<s_f\leq \max_a p_{Y|a}$. 

Sub-case (3.i): 
In this case, we have $s^\dag=c+\ep$ and $s_f\leq c$. Then,
\begin{eqnarray*}
R_c(f^\dag)-R_c(f)&=& 
\sum_{a\in\mathcal{A}}p_a(c-s^\dag)\int f^\dag(x,a)dP_{X|a}(x)
 -\sum_{a\in\mathcal{A}}p_a(c-s_f)\int f(x,a)dP_{X|a}(x)\\
 &\leq& -\ep \sum_{a\in\mathcal{A}}p_a P(\eta_a(X)>c+\ep|A=a)<0
\end{eqnarray*}

Sub-case (3.ii):
 In this case, we partition $\mathcal{A}=\mathcal{A}_1\cup\mathcal{A}_2$ into the sets $\mathcal{A}_1=\{a: p_{Y|a}\leq s_f\}$ and $\mathcal{A}_2=\{a: p_{Y|a}>s_f\}$. 
Denoting $s_a^\flat=\max\,(s_f, p_{Y|a})$,
it is clear that $s_f\leq s^\flat_a\leq s^\dag$. According to Lemma \ref{lem2}, there exist combinations $( t^\flat_a,\tau^\flat_a)_{a=1}^{|\mathcal{A}|}$ such that 
$$g_a({t}^\flat_{a},\tau^\flat_{a})=s^\flat_a, \ \ \ \ a\in\mathcal{A}.$$
We now consider  the classifier $f^\flat$ defined for all $x\in\X$ and $a\in\A$ by
$$f^\flat(x,a)=I(\eta_a(x)>t^\flat_a)+\tau^\flat_a I(\eta_a(x)=t^\flat_a).$$
Again, we  follow the construction in Lemma \ref{lem2} to set
$t_a^\flat=\sup \{t: g_a(t)<s_a^\flat\}$, and let $\tau_a(x)\equiv \tau_a^\flat$ be a constant function. Moreover, we set $\tau_a^\flat= 1$ whenever $P(\eta_a(X)>t_a^\flat|A=a)=0$ or  $P(\eta_a(X)=t_a^\flat|A=a)=0$.

Note that $s_f=s_a^\flat>c$ for $a\in\mathcal{A}_1$. Following the same argument as in case (2), we have, 
\begin{equation}\label{matha1}
    \sum_{a\in\mathcal{A}_1}(c-s_a^\flat) \int f^\flat(x,a)dP_{X|a}(x)- \sum_{a\in\mathcal{A}_1}(c-s_f)f(X,A)dP_{X|a}(x)<0.
\end{equation}
For $a\in\mathcal{A}_2$, we have $s_a^\flat=p_{Y|a}$, which implies that $(t_a^\flat,\tau_a^\flat)=(0,1)$. As a consequence, for $a\in\mathcal{A}_2$,
\begin{eqnarray}\label{matha2}
 \nonumber&&(c-s_a^\flat)\int f^\flat(x,a)dP_{X|a}(x)-(c-s_f)\int f(x,a)dP_{X|a}(x)\\
 &=&(c-p_{Y|a})-(c-s_f)\int f(x,a)dP_{X|a}(x)< 0
\end{eqnarray}
since $\int f(x,a)dP_{X|a}(x)\leq 1$ and $c-p_{Y|a}<c- s_f\leq 0$.
Combining \eqref{matha1} and \eqref{matha2} shows that $R_c(f^\flat)-R_c(f)$ equals
\begin{eqnarray*}
\sum_{a\in\mathcal{A}}p_a(c-s^\flat)\int f^\flat(x,a)dP_{X|a}(x)
-\sum_{a\in\mathcal{A}}p_a(c-s_f)\int f(x,a)dP_{X|a}(x)< 0.
\end{eqnarray*}
Now, under the Condition \ref{suff-con}, we have, for $a\in\mathcal{A}$,
$$ g_a(c,1)\geq\max_ap_{Y|a}=s^\dag=g_a(t_a^\dag,\tau_a^\dag) \geq s_a^\flat= g_a(t_a^\flat,\tau_a^\flat).$$ 
From \eqref{monotau}, we have
$t_a^\flat\leq t_a^\dag \leq c.$
Thus, $ R_c(f^\dag)-R_c(f^\flat)$ equals
\begin{eqnarray*}
&&\sum_{a\in\mathcal{A}}p_a\int (c-\eta_a(x))[f^\dag(x,a)-f^\flat(x,a)]dP_{X|a}(x)\\
&=&\sum_{a\in\mathcal{A}}p_a\int(\eta_a(x)-c)[f^\flat(x,a)-f^\dag(x,a)]dP_{X|a}(x)\\
&\leq& \sum_{a\in\mathcal{A}}p_a\int (\eta_a(x)-c)I({t_a^\flat\leq \eta_a(x)\leq t^\dag_a})dP_{X|a}(x)\leq 0.
\end{eqnarray*}
As a result,
$$ R_c(f^\dag)-R_c(f)= R_c(f^\dag)-R_c(f^\flat)+ R_c(f^\flat)-R_c(f)< 0.$$

This finishes the proof.

\section{Proof of Theorem \ref{thm-imGWT}}

Let $f_G$ be any GWTR, say of the form
$$f_{G}(x,a)=I(\eta_a(x)>t_{G,a})+\tau_{G,a}(x)I(\eta_a(x)=t_{G,a}),$$ 
satisfying predictive parity with
$$P(Y=1|\hat{Y}_{f_{G}(x,a)})=1,A=a)=s_{{G}}, \ \ \ \text{ for } a\in\mathcal{A}.$$
According to Lemma \ref{lem1}, we have $p_{Y|A=0}=P(Y=1|\eta_0(x)\geq 0,A=0)\leq s_{f_{G}}$.
By the definition of $t_1$,
we have $c<t_{1}\leq t_{G_1}$. Thus, $P(c<\eta_1(X)<t_{G_1}|A=1)>P(c<\eta_1(X)<t_{1}|A=1)>0$.

Denote $s_{NG}=P(Y=1|\eta_1(X)\geq c,A=1)$. We have $s_{G}\geq s_{NG}$, since $t_{G,1}>c$. Further, following the same argument as in Lemma \ref{lem2}, there exist $(t_0,\tau_0)$ such that
$$
\frac{P_{Y|a} [P(\eta_0(X)<t_0|A=0,Y=1)+\tau_0P(\eta_0(X)=t_0|A=0,Y=1)]}{P(\eta_0(X)<t_0|A=0)+\tau_0P(\eta_0(X)=t_0|A=0)}=s_{NG}.
$$
We consider the following classifier $f_{NG}$, which is not a GWTR:
\begin{equation}
  f_{NG}(x,a)=\left\{  \begin{array}{lcc}
    I(\eta_a(x)\geq c), && a=1;\\
    I(\eta_a(x)< t_0)+\tau_0I(\eta_a(x)= t_0) ,&& a=0.
    \end{array}\right.
\end{equation}

By construction,  $f_{NG}$ satisfies predictive parity. Moreover, when $p_1>\frac{2}{2+\delta_1\delta_2}$, we have
\begin{align*}
&R_c( f_{G})-R_c(f_{NG})=p_1\int (c-\eta_1(x))[f_{G}(x,1)- f_{NG}(x,1)]dP_{X|1}(x)\\
&+p_0\int (c-\eta_0(x))[f_{G}(x,1)- f_{NG}(x,0)]dP_{X|0}(x)\\
&\geq p_1\int_{c<\eta_1(x)<t_{G_1}} (c-\eta_1(x))dP_{X|1}(x)-2p_0
\geq p_1\int_{c+\delta_2<\eta_1(x)<t_{G_1}} (c-\eta_1(x))dP_{X|1}(x)-2p_0\\
&\geq \delta_1\delta_2p_1-2(1-p_1)>0.
\end{align*}
Thus, we have constructed a classifier that is not a GWTR satisfying predictive parity and achieving a smaller cost-sensitive risk than any fair GWTR. We can conclude that no fair Bayes-optimal classifier under predictive parity is a GWTR.

\section{Fair and Unconstrained Bayes-optimal Classifiers of the Synthetic Model}\label{appendix-synth}
In this section, we derive the unconstrained and fair Bayes-optimal classifiers for our synthetic model used in Section \ref{synth}. 
Consider the following data distribution for $(X,A,Y)$ where $A\in\{0,1\}$, $Y\in\{0,1\}$ with

\begin{itemize}
    \item For $a\in\{0,1\}$, $P(A=a)=p_a$ and $P(Y=1|A=a)=1-P(Y=0|A=a)=p_{Y|a}$;
    \item For $(a,y)\in\{0,1\}^2$, $X|A=a,Y=y\sim \mathcal{N}(\mu_{a,y},\sigma^2I_2)$ with $\mu_{a,y}=(2a-1,2y-1)^\top$.
\end{itemize}

Denote by $g_{a,y}(x)=\frac{1}{2\pi\sigma^2}\exp(-\frac1{2\sigma^2}\|x-\mu_{a,y}\|^2)$ the conditional density function of $X$ given $A=a$ and $Y=y$.  we have
\begin{eqnarray*}
   \eta_a(x)&=&P(Y=1|X=x,A=a)
   =\frac{p_{Y|a}g_{a,1}(x) }{p_{Y|a} g_{a,1}(x)+(1-p_{Y|a})g_{a,0}(x)}\\
   &=&\frac{p_{Y|a}\exp(-\frac1{2\sigma^2}\|x-\mu_{a,1}\|^2)}{p_{Y|a}\exp(-\frac1{2\sigma^2}\|x-\mu_{a,1}\|^2)+(1-p_{Y|a})\exp(-\frac1{2\sigma^2}\|x-\mu_{a,0}\|^2)}.
\end{eqnarray*}
Then, the unconstrained deterministic Bayes-optimal classifier $f^\star$ is
\begin{eqnarray*}
f^\star(x,a)&=&I(\eta_a(x)>c)\\
&=& I \left((1-c){p_{Y|a}\exp(-\frac1{2\sigma^2}\|x-\mu_{a,1}\|^2)>c(1-p_{Y|a})\exp(-\frac1{2\sigma^2}\|x-\mu_{a,0}\|^2)}\right)\\
&=& I\left(x^\top(\mu_{a,0}-\mu_{a,1})<\log\frac{(1-c)p_{Y|a}}{c(1-p_{Y|a})}\right).
\end{eqnarray*}

For given $p_{Y|A=0}$, Condition \ref{suff-con} is equivalent to
\begin{eqnarray}\label{suff-gaussian}
\nonumber p_{Y|A=1}&\leq& P(Y=1|\eta_A(X)>c,A=0)
\frac{p_{Y|A=0}P(\eta_A(X)>c|A=0,Y=1)}{P(\eta_A(X)>c|A=0)}\\
\nonumber&=&\frac{p_{Y|A=0}P\left(X^\top(\mu_{0,0}-\mu_{0,1})<\log\frac{(1-c)p_{Y|A=0}}{c(1-p_{Y|A=0})}|A=0,Y=1\right)}
{P\left(X^\top(\mu_{0,0}-\mu_{0,1})<\log\frac{(1-c)p_{Y|A=0}}{c(1-p_{Y|A=0})}|A=0\right)}\\
&=&
\frac{p_{Y|A=0}\bar\Phi\left(\frac{\sigma\log(q_0(c))}{2}-\frac{1}{\sigma}\right)}{p_{Y|A=0}\bar\Phi\left(\frac{\sigma\log(q_0(c))}{2}-\frac{1}{\sigma}\right)+(1-p_{Y|A=0})\bar\Phi\left(\frac{\sigma\log(q_0(c))}{2}+\frac{1}{\sigma}\right)},
\end{eqnarray}
where $q_a(c)=\frac{c(1-p_{Y|a})}{(1-c)p_{Y|a}}$ and $\bar\Phi(t)=1-\Phi(t)$ with $\Phi(t)$ the cumulative distribution function of the standard normal distribution.

Now we consider fair Bayes optimal classifiers under \eqref{suff-gaussian}.
We consider the GWTR $f_{t_1,t_0}$ such that for $a\in\{0,1\}$ and all $x\in\X$,
$f_{t_1,t_0}(x,a)=I(\eta_a(x)>t_a)$.
Following the same argument as in \eqref{suff-gaussian}, we have
\begin{equation*}
P(Y=1|\eta_a(x)>t_a,A=a)=\frac{p_{Y|a}\bar\Phi\left(\frac{\sigma\log(q_a(t_a))}{2}-\frac{1}{\sigma}\right)}{p_{Y|a}\bar\Phi\left(\frac{\sigma\log(q_a(t_a))}{2}-\frac{1}{\sigma}\right)+(1-p_{Y|a})\bar\Phi\left(\frac{\sigma\log(q_a(t_a))}{2}+\frac{1}{\sigma}\right)}.
\end{equation*}
Then, $f_{t_1,t_0}$ satisfies predictive parity if
\begin{equation*}
\frac{p_{Y|A=1}\bar\Phi\left(\frac{\sigma\log(q_1(t_1))}{2}-\frac{1}{\sigma}\right)}{(1-p_{Y|A=1})\bar\Phi\left(\frac{\sigma\log(q_1(t_1))}{2}+\frac{1}{\sigma}\right)}=\frac{p_{Y|A=0}\bar\Phi\left(\frac{\sigma\log(q_0(t_0))}{2}-\frac{1}{\sigma}\right)}{(1-p_{Y|A=0})\bar\Phi\left(\frac{\sigma\log(q_0(t_0))}{2}+\frac{1}{\sigma}\right)}.
\end{equation*}
Note that, as a function of $t_a$, $t_a\mapsto P(Y=1|\eta_a(x)>t_a,A=a)$ is strictly monotone increasing. Thus, for $T_1(t)=t$, there exists a function $t\mapsto T_0(t)$ such that
\begin{equation*}
\frac{p_{Y|A=1}\bar\Phi\left(\frac{\sigma\log(q_1(t))}{2}-\frac{1}{\sigma}\right)}{(1-p_{Y|A=1})\bar\Phi\left(\frac{\sigma\log(q_1(t))}{2}+\frac{1}{\sigma}\right)}=\frac{p_{Y|A=0}\bar\Phi\left(\frac{\sigma\log(q_0(T_0(t)))}{2}-\frac{1}{\sigma}\right)}{(1-p_{Y|A=0})\bar\Phi\left(\frac{\sigma\log(q_0(T_0(t)))}{2}+\frac{1}{\sigma}\right)}.
\end{equation*}
Then $f_{T_1(t),T_0(t)}$ satisfies predictive parity and its cost-sensitive risk $R_c(f_{T_1(t),T_0(t)})$ is 
\begin{eqnarray*}
&&\sum_{a\in\{0,1\}} cp_a(1-p_{Y|a})P(\eta_a(X)\ge  t_a|A=a,Y=0)\\
&&\qquad\qquad\qquad+ \sum_{a\in\{0,1\}} (1-c)p_ap_{Y|a}P(\eta_a(X)<t_a|A=a,Y=1)\\
&=&\sum_{a\in\{0,1\}}(1-c)p_ap_{Y|a}\Phi\left(\frac{\sigma\log(q_a(T_a(t)))}{2}-\frac{1}{\sigma}\right)\\
&&\qquad\qquad\qquad
+
\sum_{a\in\{0,1\}}cp_a(1-p_{Y|a})\bar\Phi\left(\frac{\sigma\log(q_a(T_a(t)))}{2}+\frac{1}{\sigma}\right).
\end{eqnarray*}
Let $t^\star$ be defined as
$$t^\star=\underset{t\in[0,1]}{\text{argmin}}\,R_c(f_{T_1(t),T_0(t)}).$$
Thus, under \eqref{suff-gaussian}, the fair Bayes-optimal classifier under predictive parity is given by $f_{t^\star,T_0(t^\star)}$. This classifier can be computed numerically as both $T_0(t)$ and $t^\star$ can be found numerically.

\section{Experimental Settings and More Simulation Results}\label{sec:training_details}
{\bf Training details.} Our experiments are conducted on a personal computer with an Intel(R) Core(TM) i9-9920X CPU @ 3.50Ghz and an NVIDIA GeForce RTX 2080 Ti GPU. For the Adult and COMPAS datasets, we employ  the same training settings as in \cite{CHS2020}. 
We train the conditional probability predictor using a three-layer fully connected net with 32 neurons in the hidden  layers.  
For the CelebA dateset, we adopt the same settings in \cite{WQK2020} to train the conditional probability predictor with ResNet-50, pre-trained on the ImageNet dataset. 
We also apply the dropout technique with  $p=0.5$ to improve the model performance.
In all the simulations, we use the Adam optimizer with the default parameters.  The details are summarized in Table \ref{detail}.

\begin{table}[ht]
\caption{Training details for three datasets}
\label{detail}
\vskip 0.1in
\begin{center}
\begin{small}
\begin{sc}
\begin{tabular}{c|c|c|c}\hline
  Dataset        & Adult Census&  COMPAS & Celeba \\\hline
 Batch size      &512          &  2048   &    32     \\
Training Epochs  &  200        & 500     &    50     \\
Optimizer        & Adam    &   Adam  &Adam \\
Learning rate    &  1e-1       &5e-4     &   1e-4       \\
Pre-Training       &  N/A        & N/A    &  ImageNet \\
Dropout      &  N/A        & N/A    &  0.5 
\\\hline
\end{tabular}
\end{sc}
\end{small}
\end{center}
\vskip -0.1in
\end{table}


\begin{figure}[b]
    \centering
    \includegraphics[scale=0.36]{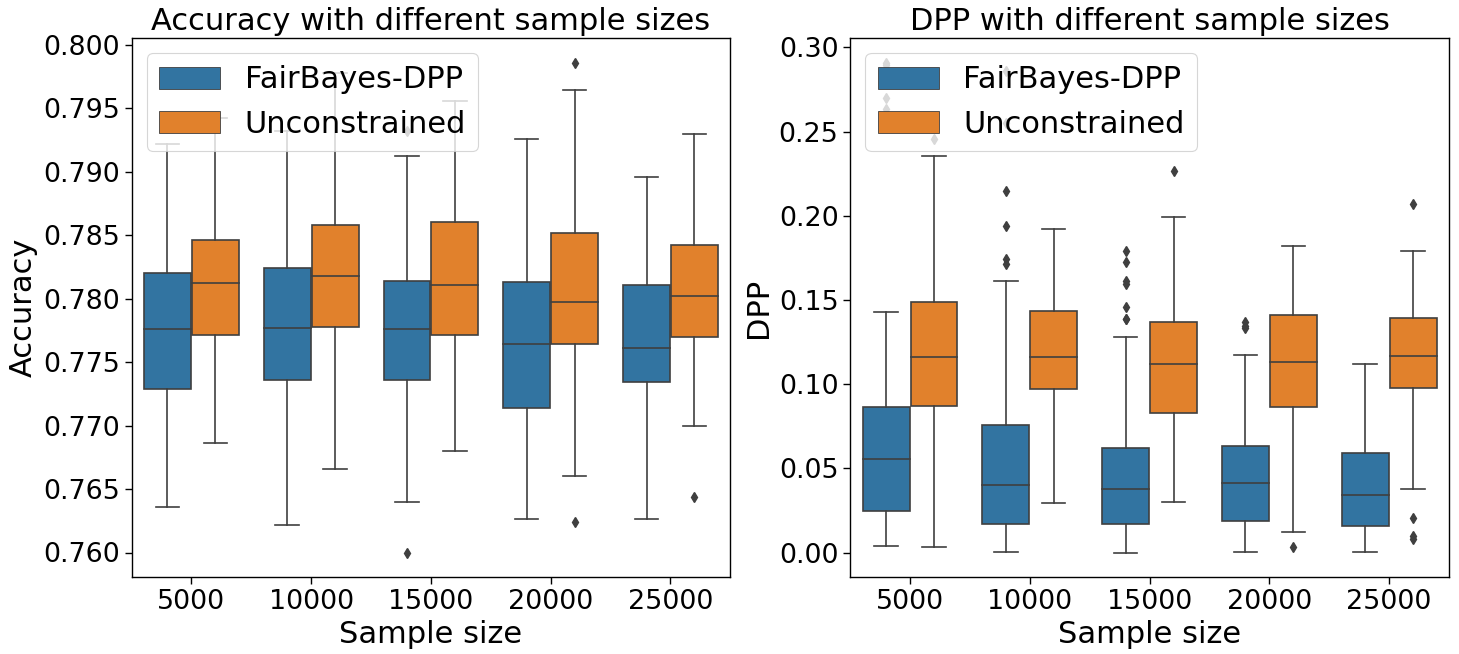}
        \vspace{-0.2cm}
    \caption{Accuracy and DPP as a function of sample size.}
    \vspace{-0.3cm}
    \label{fig-sample-size}
\end{figure}

\subsection{Synthetic Data}
We conduct more experiments to evaluate the performance of our FairBayes-DPP algorithm under different model and training settings. 
We consider the same synthetic model as in Section \ref{synth} with different settings on sample size, proportion $P(A=0)$ of the minority group and cost parameters. We also extend the synthetic model to a multi-class protected attribute. 
In all scenarios, we repeat the experiments 100 times\footnote{The randomness of the experiment comes from the random generation of the training and test data sets.}.
\subsubsection{Sample Size}

We first evaluate FairBayes-DPP with different sample sizes.
In the experiment, we fix $c=0.5$, $p(A=1)=0.3$, $p(Y=1|A=1)=0.6$ and $p(Y=1|A=1)=0.2$. We further fix the number of test data points to be $5000$, and change the number of training data points from $5000$ to $25000$. 
The simulation results are presented in  Figure \ref{fig-sample-size}. 
It can be seen that FairBayes-DPP has a smaller disparity than the unconstrained classifier. 
As the sample size grows, the performance of FairBayes-DPP improves, since the estimation error reduces with more training data points.

\begin{figure}[t]
    \centering
    \includegraphics[scale=0.36]{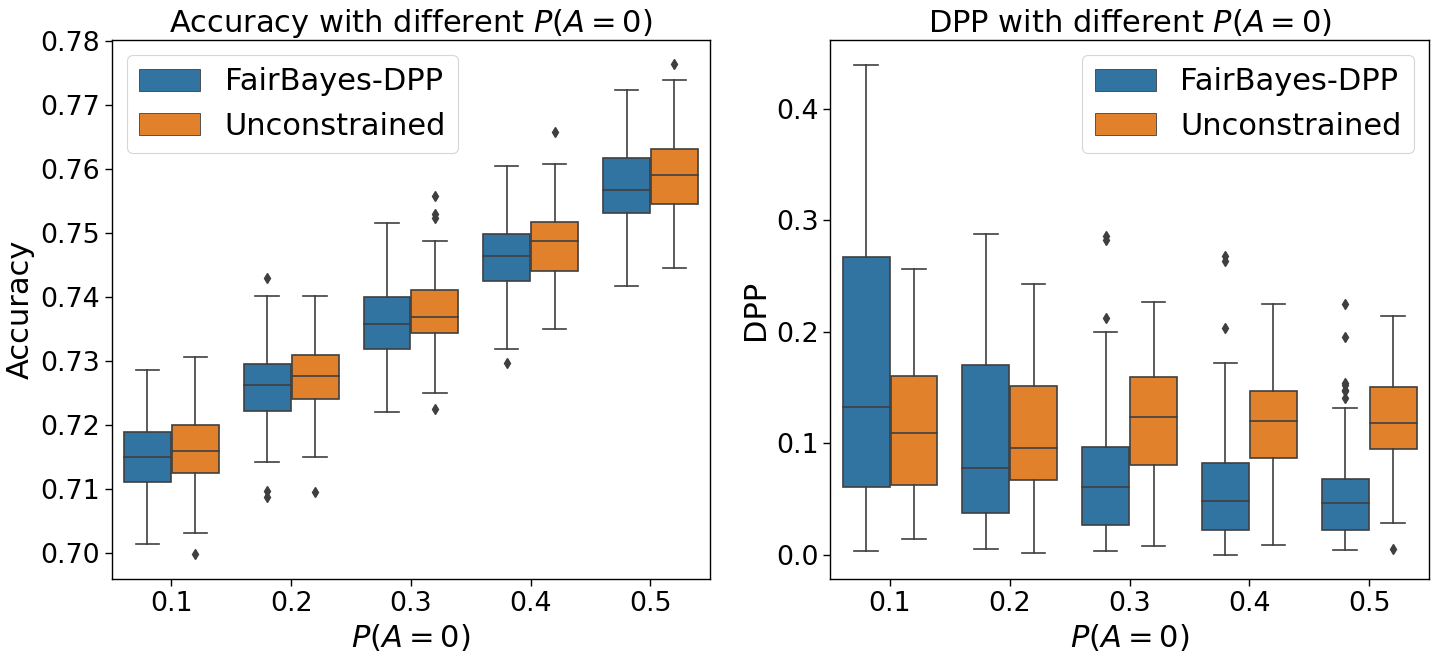}
        \vspace{-0.2cm}
    \caption{Accuracy and DPP as a function of $P(A=0).$}
    \vspace{-0.3cm}
    \label{fig-pa}
\end{figure}

\begin{figure}[t]
    \centering
    \includegraphics[scale=0.36]{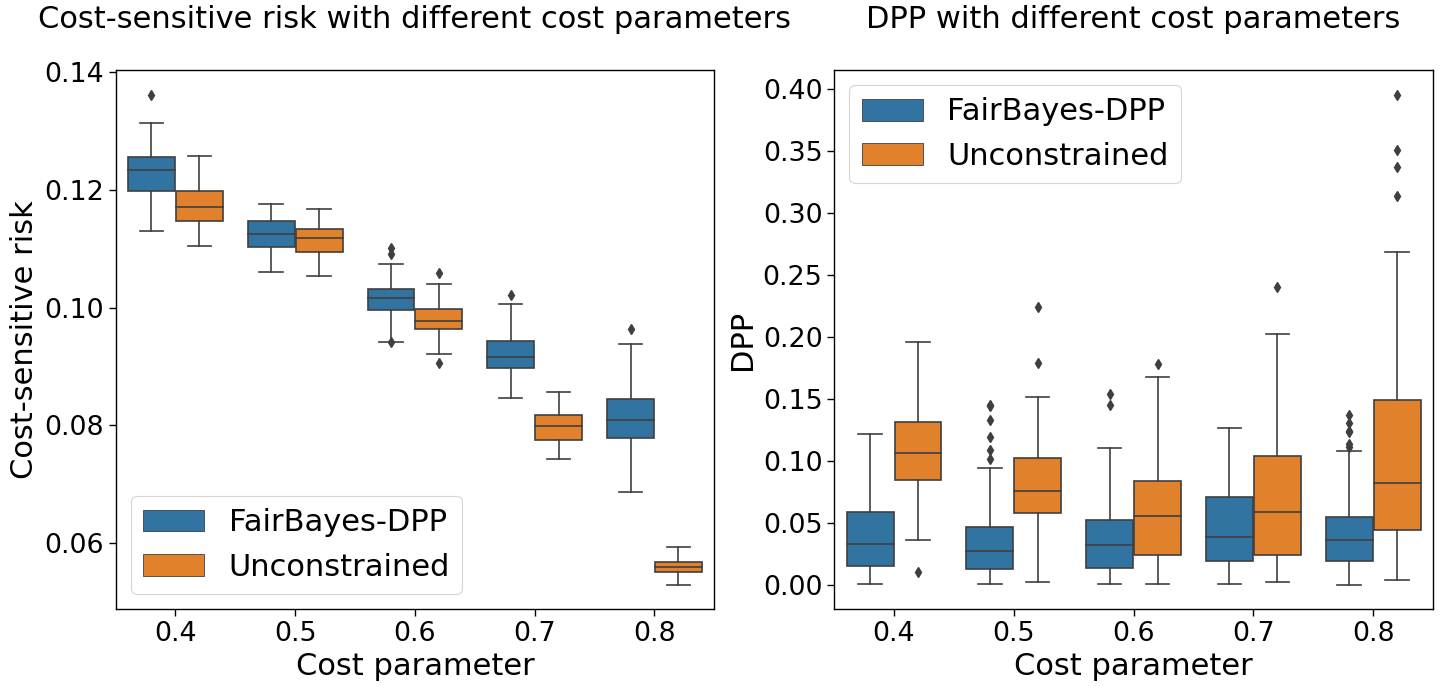}
        \vspace{-0.2cm}
    \caption{Cost-sensitive risk and DPP as a function of the cost parameter.}
    \vspace{-0.3cm}
    \label{fig-cost}
\end{figure}
\subsubsection{Proportion of Minority Group}\label{subsec:minority}
Next, we evaluate the effect of the proportion $P(A=0)$ of the minority group on the performance of FairBayes-DPP. 
We fix $c=0.5$, $p(Y=1|A=1)=0.6$, $p(Y=1|A=0)=0.2$, and vary $P(A=0)$
from $0.5$ to $0.9$. 
Moreover, we set the training data size and test data size to be $25000$ and $5000$, respectively. 
Figure \ref{fig-pa} presents the simulation results. 

We observe that, for both FairBayes-DPP and unconstrained learning, the test accuracy increases with $P(A=0)$. 
The sample complexity of learning the unconstrained classifier should intuitively depend on the sample size of the smallest group. 
When $P(A=0)$ is very small, the estimator of $\eta_0$ has large variability and results in a small test accuracy. 

We also observe that the performance of FairBayes-DPP is unstable when $P(A=0)$ is very small. This {limitation} is caused by the unstable estimation of $\eta_0$, which is used by  FairBayes-DPP to adjusts the per-class thresholds.  
As we can see, the performance of FairBayes-DPP improves rapidly when $P(A=0)$ grows. We emphasize that the success of  FairBayes-DPP relies on the consistent estimation of the per-group feature-conditional probabilities of the labels.

\subsubsection{Cost Parameter}
 We then evaluate the effect of cost parameter $c$. We fix $P(Y=1)=0.3$, $p(Y=1|A=1)=0.5$, $p(Y=1|A=0)=0.2$, and vary $c$ from $0.4$ to $0.8$.  
 Again, we set the training and test data sizes to be $25000$ and $5000$, respectively. We present the simulation results  in Figure \ref{fig-pa}. 
We  observe that  FairBayes-DPP successfully mitigates disparity with a wide range of cost parameters.

\begin{table}[b]
\caption{Parameters of synthetic model for multi-clase protected attribute.}
\label{model-para}
\vskip 0.1in
\begin{center}
\begin{small}
\begin{sc}
\begin{tabular}{c|ccccc}\hline
 \multicolumn{6}{c}{$|\mathcal{A}=3|$}\\\hline
  $a$          &     1       &      2   &    3     &       &    \\\hline
 $p_a$         &    0.3      &     0.3  &    0.4   &         &     \\
$p_{Y|a}$      &    0.2      &     0.6  &    0.3   &         &     \\\hline
 \multicolumn{6}{c}{$|\mathcal{A}=5|$}\\\hline
  $a$          &     1       &      2   &    3     &   4    &  5  \\\hline
 $p_a$         &    0.2      &     0.3  &    0.2   &    0.15     &   0.15  \\
$p_{Y|a}$      &    0.2      &     0.6  &    0.3   &    0.4     & 0.2    \\\hline\end{tabular}
\end{sc}
\end{small}
\end{center}
\vskip -0.1in
\end{table}

\subsubsection{Multi-class Protected Attribute}
Finally, we study a multi-class protected attribute. 
We generate data $a\in\mathcal{A}=\{1,2,...,|\mathcal{A}|\}$ and $y\in\{0,1\}$ by setting $\mu_{ay}=(2y-1)e_a$, where $e_a\in\mathbb{R}^{|\mathcal{A}|}$ the unit vector with the $a$-th element equal to unity.
Conditional on $A=a$ and $Y=y$, $X$ is generated from a multivariate
Gaussian distribution $N(\mu_{ay},2^2I_{|\mathcal{A}|})$.

We consider two cases, $|\mathcal{A}|=3$ and $|\mathcal{A}|=5$, with the model parameters presented in Table \ref{model-para}. For both cases, we set $c=0.5$, the training data sample size as $50000$ and the test data sample size as $5000$. 
We present the simulation results in Figure \ref{fig-multia}. Again, FairBayes-DPP
achieves superior performance in preserving accuracy and mitigating bias.

\begin{figure}[t]
    \centering
    \includegraphics[scale=0.36]{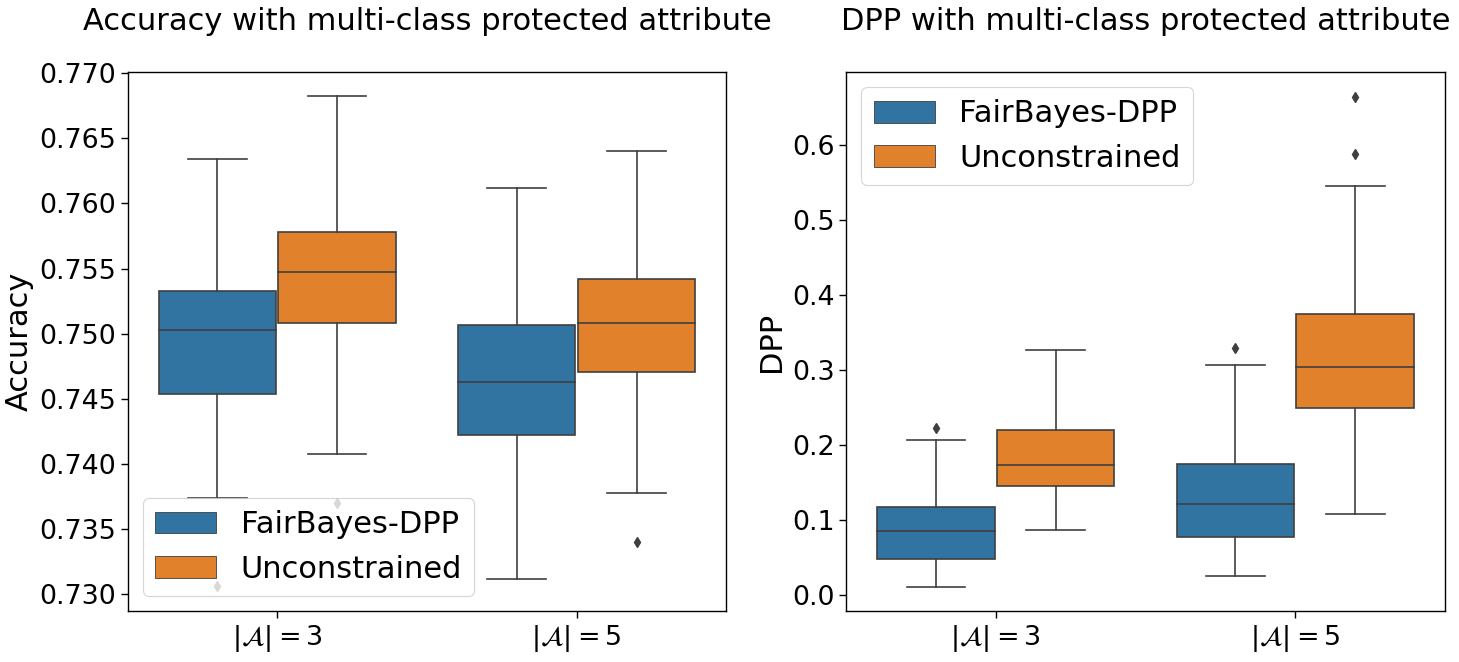}
        \vspace{-0.2cm}
    \caption{Accuracy and DPP with multi-class protected attribute.}
    \vspace{-0.3cm}
    \label{fig-multia}
\end{figure}

\subsection{CelebA Dataset}
In the main text, we have presented the simulation results for the first six attributes of the CelebA dataset. 
Here, we show the simulation results for the remaining 20 attributes in Table \ref{Table-celeba-re}. 
Again, we observe that FairBayes-DPP mitigates the gender bias effectively in most cases, and preserves model accuracy.

\begin{table}[ht]
 \caption{Per-attribute accuracy and DPP of the remaining 20 attributes from the CelebA dataset.} 
        \label{Table-celeba-re}
\begin{center}
\setlength{\tabcolsep}{9.45pt}
\renewcommand{\arraystretch}{1.1}
\begin{small}
\begin{sc}
    \begin{tabular}{l|cc|cc}\hline
 \multirow{3}[1]{*}{Attributes}   &      \multicolumn{2}{c|}{Per-attribute Accuracy}&      \multicolumn{2}{c}{Per-attribute DPP}\\\cline{2-5}
                      & FairBayes& Uncon- & FairBayes& Uncon- \\
                      & -DPP& strained & -DPP& strained \\\hline
 Black Hair           & 0.895(0.004)   & 0.899(0.003)   & 0.023(0.009)   & 0.033(0.013) \\
 Blond Hair           & 0.958(0.001)   & 0.959(0.001)   & 0.028(0.014)   & 0.119(0.042) \\
 Blurry               & 0.963(0.001)   & 0.963(0.001)   & 0.023(0.017)   & 0.047(0.017) \\
 Brown Hair           & 0.886(0.003)   & 0.889(0.004)   & 0.029(0.009)   & 0.078(0.028) \\
 Bushy Eyebrows       & 0.928(0.001)   & 0.926(0.001)   & 0.055(0.030)   & 0.166(0.038) \\
 Chubby               & 0.957(0.002)   & 0.957(0.002)   & 0.032(0.012)   & 0.043(0.026) \\
 Eyeglasses           & 0.996(0.000)   & 0.997(0.000)   & 0.010(0.005)   & 0.004(0.003) \\
 High Cheekbones      & 0.875(0.002)   & 0.876(0.002)   & 0.044(0.008)   & 0.143(0.016) \\
 Mouth Slightly Open  & 0.940(0.001)   & 0.940(0.001)   & 0.011(0.003)   & 0.017(0.008) \\
 Narrow Eyes          & 0.873(0.002)   & 0.875(0.003)   & 0.110(0.025)   & 0.063(0.026) \\
 Oval Face            & 0.756(0.002)   & 0.756(0.003)   & 0.033(0.016)   & 0.108(0.031) \\
 Pale Skin            & 0.970(0.001)   & 0.970(0.001)   & 0.059(0.040)   & 0.111(0.034) \\
 Pointy Nose          & 0.775(0.003)   & 0.774(0.003)   & 0.032(0.018)   & 0.063(0.022) \\
 Receding Hairline    & 0.939(0.001)   & 0.938(0.001)   & 0.067(0.019)   & 0.036(0.034) \\
 Smiling              & 0.928(0.001)   & 0.928(0.002)   & 0.021(0.005)   & 0.046(0.014) \\
 Straight Hair        & 0.842(0.002)   & 0.842(0.003)   & 0.056(0.007)   & 0.020(0.013) \\
 Wavy Hair            & 0.844(0.003)   & 0.847(0.003)   & 0.019(0.014)   & 0.087(0.021) \\
 Wearing Earrings     & 0.889(0.027)   & 0.908(0.001)   & 0.075(0.050)   & 0.207(0.037) \\
 Wearing Hat          & 0.991(0.000)   & 0.991(0.000)   & 0.012(0.013)   & 0.047(0.018) \\
 Wearing Necklace     & 0.868(0.002)   & 0.868(0.001)   & 0.077(0.047)   & 0.069(0.052) \\\hline
  \end{tabular}
  \end{sc}
\end{small}
\end{center}
  \vspace{-0.3cm}
\end{table}

\end{document}